\documentclass[letterpaper]{article} 
\usepackage{aaai24}  
\usepackage{times}  
\usepackage{helvet}  
\usepackage{courier}  
\usepackage[hyphens]{url}  
\usepackage{graphicx} 
\usepackage{natbib}  
\usepackage{caption} 
\usepackage{algorithm}
\usepackage{algorithmic}

\usepackage[utf8]{inputenc} 
\usepackage{url}            
\usepackage{booktabs}       
\usepackage{amsfonts}       
\usepackage{nicefrac}       
\usepackage{microtype}      
\usepackage{xcolor}         
\usepackage{enumerate}

\usepackage{amsmath}
\usepackage{amssymb}
\usepackage{mathtools}
\usepackage{amsthm}

\usepackage{cleveref}

\newtheorem{innercustomlemma}{Lemma}
\newenvironment{customlemma}[1]
  {\renewcommand\theinnercustomlemma{#1}\innercustomlemma}
  {\endinnercustomlemma}

\newtheorem{innercustomthm}{Theorem}
\newenvironment{customthm}[1]
  {\renewcommand\theinnercustomthm{#1}\innercustomthm}
  {\endinnercustomthm}

\newtheorem{innercustomcor}{Corollary}
\newenvironment{customcor}[1]
  {\renewcommand\theinnercustomcor{#1}\innercustomcor}
  {\endinnercustomcor}

\newtheorem{innercustompro}{Proposition}
\newenvironment{custompro}[1]
  {\renewcommand\theinnercustompro{#1}\innercustompro}
  {\endinnercustompro}

\newtheorem{remark}{Remark}
\theoremstyle{plain}
\newtheorem{theorem}{Theorem}
\newtheorem{proposition}[theorem]{Proposition}
\newtheorem{lemma}[theorem]{Lemma}
\newtheorem{corollary}[theorem]{Corollary}
\theoremstyle{definition}
\newtheorem{definition}[theorem]{Definition}

\newtheorem{example}{Example} 

\newcommand{\normmm}[1]{{\left\vert\kern-0.25ex\left\vert\kern-0.25ex\left\vert #1 
   \right\vert\kern-0.25ex\right\vert\kern-0.25ex\right\vert}}

\newcommand{\sroman}[1]{(\text{\romannumeral#1})}

\newcommand{\Ceil}[1]{\lceil #1 \rceil}

\usepackage{diagbox}
\usepackage{multirow}
\usepackage{adjustbox}

\usepackage{graphicx}

\usepackage{bm}

\usepackage{newfloat}
\usepackage{listings}
\DeclareCaptionStyle{ruled}{labelfont=normalfont,labelsep=colon,strut=off} 
\floatstyle{ruled}
\newfloat{listing}{tb}{lst}{}
\floatname{listing}{Listing}
\title{Coverage-Guaranteed Prediction Sets for Out-of-Distribution Data}
\author{
    Xin Zou,
    Weiwei Liu\thanks{corresponding author}
}
\affiliations{
    School of Computer Science,\\
    Institute of Artificial Intelligence,\\
    National Engineering Research Center for Multimedia Software,\\
    Hubei Key Laboratory of Multimedia and Network Communication Engineering,\\
    Wuhan University, China

    \{zouxin2021, liuweiwei863\}@gmail.com
}

\usepackage{bibentry}

\begin{document}

\maketitle

\begin{abstract}
Out-of-distribution (OOD) generalization has attracted increasing research attention in recent years, due to its promising experimental results in real-world applications. In this paper, we study the confidence set prediction problem in the OOD generalization setting. Split conformal prediction (SCP) is an efficient framework for handling the confidence set prediction problem. However, the validity of SCP requires the examples to be exchangeable, which is violated in the OOD setting. Empirically, we show that trivially applying SCP results in a failure to maintain the marginal coverage when the unseen target domain is different from the source domain. To address this issue, we develop a method for forming confident prediction sets in the OOD setting and theoretically prove the validity of our method. Finally, we conduct experiments on simulated data to empirically verify the correctness of our theory and the validity of our proposed method.
\end{abstract}

\section{Introduction} \label{sec::introduction}
Recent years have witnessed the remarkable success of modern machine learning techniques in many applications. A fundamental assumption of most machine learning algorithms is that the training and test data are drawn from the same underlying distribution. However, this assumption is consistently violated in many practical applications. In reality, the test environment is influenced by a range of factors, such as the distributional shifts across photos caused by the use of different cameras in image classification tasks, the voices of different persons in voice recognition tasks, and the variations between scenes in self-driving tasks \cite{DBLP:conf/iclr/NagarajanAN21}. As a result, there is now a rapidly growing body of research with a focus on generalizing to unseen target domains with the help of the source domains, namely OOD generalization \cite{DBLP:journals/corr/abs-2108-13624}.

Existing OOD generalization methods focus on improving worst-case performance on the target domains, i.e., improving the average test accuracy of the model on the worst target domain. However, in some systems that require high security (such as medical diagnosis), even a single mistake may have disastrous consequences. In these cases, it is important to quantify the uncertainty of the predictions. One way to perform uncertainty estimation \citep{DBLP:journals/corr/AmodeiOSCSM16, DBLP:journals/jamia/JiangOKO12, DBLP:conf/nips/JiangKGG18, DBLP:conf/iclr/AngelopoulosBJM21} is to create confident prediction sets that provably contain the correct answer with high probability. Let $X_{n+1} \in \mathcal{X}$ be a new test example for which we would like to predict the corresponding label $Y_{n+1} \in \mathcal{Y}$, where $\mathcal{X}$ is the input space and $\mathcal{Y}$ is the label space. For any given $\alpha \in (0,1)$, the aim of confidence set prediction is to construct a set-valued output, $\mathcal{C}(X_{n+1})$, which contains the label $Y_{n+1}$ with distribution-free marginal coverage at a significance level $\alpha$, i.e., $\mathbb{P}\left( Y_{n+1} \in \mathcal{C}(X_{n+1}) \right) \ge 1-\alpha$. A confidence set predictor $\mathcal{C}^\alpha$ is said to be \textbf{valid} if $\mathbb{P}\left( Y_{n+1} \in \mathcal{C}^\alpha(X_{n+1}) \right) \ge 1-\alpha$ for any $\alpha \in (0,1)$, where $\alpha$ is a hyper-parameter of the predictor. To simplify the notation, we omit the superscript $\alpha$ in the remainder of this paper.

Conformal prediction \citep[\textbf{CP}]{10.5555/1062391, DBLP:journals/jmlr/ShaferV08} is a model-agnostic, non-parametric and distribution-free (the coverage guarantee holds for any distribution) framework for creating confident prediction sets. Split conformal prediction \citep[\textbf{SCP}]{DBLP:journals/ml/Vovk13, 10.5555/1062391}, a special type of CP, has been shown to be computationally efficient. SCP reserves a set of data as the calibration set, and then uses the relative value of scores of the calibration set and that of a new test example to construct the prediction set. The validity of SCP relies on the assumption that the examples are exchangeable. However, in the OOD setting, the distributional shift between the training and test distributions leads to the violation of the exchangeability assumption. We empirically evaluate the performance of SCP in the OOD setting in Section \ref{sec::motivating-experiment}. Unfortunately, we find that trivially applying SCP results in a failure to maintain marginal coverage in the OOD setting.

To address this issue, we construct a set predictor based on the $f$-divergence \cite{fdivergence} between the test distribution (target domain) and the convex hull of the training distributions (source domains). We theoretically show that our set predictor is guaranteed to maintain the marginal coverage (\Cref{cor::correct-coverage-guarantee}). We then conduct simulation experiments to verify our theory.

The remainder of this article is structured as follows: \S \ref{sec::related-work} introduces some related works; \S \ref{sec::pre} presents the notation definitions and preliminaries; \S \ref{sec::motivating-experiment} conducts experiments that show the failure of SCP in the OOD generalization setting; \S \ref{sec::ood-scp} creates corrected confidence set predictor in the OOD generalization setting; \S \ref{sec::experiments} provides our experimental results. \S \ref{sec::discussion} make discussions with the most related work. Finally, the conclusions are presented in \S \ref{sec::conclusion}. All of our proofs are attached in Appendix A.

\section{Related Works} \label{sec::related-work}
\textbf{OOD generalization.} OOD generalization aims to train a model with data from the source domains so that it is capable of generalizing to an unseen target domain. A large number of algorithms have been developed to improve OOD generalization. One series of works focuses on minimizing the discrepancies between the source domains \citep{DBLP:conf/cvpr/MMD,DBLP:journals/jmlr/DANN,DBLP:conf/eccv/CDANN,DBLP:conf/eccv/CORAL}. Meta-learning domain generalization \citep[MLDG]{DBLP:conf/aaai/MLDG} leverages the meta-learning approach and simulates train/test distributional shift during training by synthesizing virtual testing domains within each mini-batch. Another line of works \citep{DAT, DBLP:journals/corr/MAT-LDAT} conducts adversarial training \citep{DBLP:conf/iclr/MadryMSTV18} to improve the OOD generalization performance. \citep{zou2023on} considers improving the adversarial robustness of the unseen target domain. Notably, the above works all focus on improving the average performance on the target domain; in contrast, we focus on designing valid confidence set predictors for data from the unseen target domains, as this is a crucial element of making high-stakes decisions in systems that require high security.

\textbf{Conformal prediction.} As introduced in \S 1, conformal prediction is a model-agnostic, non-parametric, and distribution-free framework that provides valid confidence set predictors. Generally speaking, examples are assumed to be exchangeable in a CP context. Most pertinent to our work, \cite{DBLP:conf/iclr/GendlerWDR22, DBLP:conf/nips/TibshiraniBCR19, DBLP:conf/icml/FischSJB21, DBLP:journals/corr/abs-2008-04267, DBLP:conf/nips/GibbsC21, oliveira2022split} all consider various situations in which the exchangeability of the examples is violated to some extent. \cite{DBLP:conf/iclr/GendlerWDR22} considers the case in which the test examples may be adversarially attacked \citep{DBLP:journals/corr/SzegedyZSBEGF13, DBLP:journals/corr/GoodfellowSS14, DBLP:conf/iclr/MadryMSTV18}; \cite{DBLP:conf/nips/TibshiraniBCR19} investigates the situation in which the density ratio between the target domain and the source domain is known; \cite{DBLP:conf/icml/FischSJB21} studies the few-shot learning setting and assumes that the source domains and the target domain are independent and identically distributed (i.i.d.) from some distribution on the domains; \cite{DBLP:conf/nips/GibbsC21} considers an online learning setting and \cite{oliveira2022split} provides results when the examples are mixing \citep{klenke2013probability, chen2010nonlinearity, yu1994rates}. Different from all the works discussed above, we consider the OOD generalization setting in which the $f$-divergence between the target domain and the convex hull of the source domains is constrained. The most related work among them is \citep{DBLP:journals/corr/abs-2008-04267}, which studies the worst-case coverage guarantee of a $f$-divergence ball centered at the single source domain. For the discussions about similarities and differences with \citep{DBLP:journals/corr/abs-2008-04267}, please refer to Section \ref{sec::discussion}.

\section{Preliminaries} \label{sec::pre}
We begin with the OOD setups and a review of conformal prediction.

\textbf{Notations.} We denote $\{ 1, 2, \dots, n \}$ by $[n]$ for $n \in \mathbb{N}_+$. For a distribution $P$ on $\mathbb{R}$, we define the quantile function of $P$ as $\mathcal{Q}(\beta;P) \coloneqq \inf \{ s \in \mathbb{R}| P(S\le s) \ge \beta \}$. Similarly, for a cumulative distribution function (c.d.f.) $F$ on $\mathbb{R}$, we define $\mathcal{Q}(\beta;F) \coloneqq \inf \{ s \in \mathbb{R}| F(s) \ge \beta \}$. For $n$ distributions $P_1, \dots, P_n$, we define $\mathcal{CH}\left( P_1, \dots, P_n \right) \coloneqq \left\{ \sum_{i=1}^n \lambda_i P_i | \lambda_1, \dots, \lambda_n \ge 0; \sum_{i=1}^n \lambda_i = 1 \right\}$ as the convex hull of the distributions $P_1, \dots, P_n$. We further define $\mathcal{N}(\mu, \Sigma)$ as the multi-variable Gaussian distribution with mean vector $\mu$ and covariance matrix $\Sigma$. For a set $A$, we define the indicator function as $\mathbb{I}_A(\cdot)$, where $\mathbb{I}_A(x) = 1$ if $x \in A$ and $\mathbb{I}_A(x) = 0$ otherwise.

\subsection{Out-of-Distribution Generalization}

We define the input space as $\mathcal{X}$ and the label space as $\mathcal{Y}$. We set $\mathcal{Y} = \{ \pm 1 \}$, $\mathcal{Y} = \{ 1,2, \dots, K \}$ (where $K$ is the number of classes), and $\mathcal{Y} = \mathbb{R}$ for the binary classification problem, the multi-class classification problem, and the regression problem, respectively. Let $\mathcal{S} \coloneqq \left\{S_1, \dots, S_d \right\}$ be the set of source domains, where $d$ is the number of source domains. $S_1, \dots, S_d$ are distributions on $\mathcal{Z} \coloneqq \mathcal{X} \times \mathcal{Y}$, and we use the terminologies "domain" and "distribution" interchangeably in this paper. Let $T$ denote the target domain. The goal of OOD generalization is to obtain good performance on all $T \in \mathcal{T}$, where $\mathcal{T}$ is the set of all possible target domains; we usually assume $\mathcal{S} \subseteq \mathcal{T}$.

In a standard OOD generalization setting, we learn a predictor $h \in \mathcal{H} \subseteq \{ h: \mathcal{X} \xrightarrow{} \mathcal{Y} \}$ from the source domains $\mathcal{S}$ and define a loss function $\ell: \mathcal{Y} \times \mathcal{Y} \xrightarrow{} \mathbb{R}_*$ where $\mathbb{R}_* = [0, + \infty)$. We aim to minimize the worst-case population risk of the predictor $h$ on the unseen target domain as follows:
\begin{equation*}
    \mathcal{R}_\mathcal{T}(h) = \underset{T \in \mathcal{T}}{\max} \underset{(X,Y)\sim T}{\mathbb{E}} \left[ \ell\left( h(X), Y \right) \right].
\end{equation*}

However, in some systems that require high security, a mistake may lead to serious disasters. In these cases, a good solution is to output a prediction set with a marginal coverage guarantee. For a predefined confidence level $1-\alpha \in (0,1)$, we wish to output a prediction set $\mathcal{C}(x) \subseteq \mathcal{Y}$ such that, for any $T \in \mathcal{T}$:
\begin{equation} \label{eq::coverage-guarantee}
    \underset{(X,Y) \sim T, \mathcal{C}}{\mathbb{P}}\left[ Y \in \mathcal{C}(X) \right] \ge 1-\alpha,
\end{equation}
where the probability is over the randomness of test examples $(X,Y) \sim T$ and the randomness of the prediction set $\mathcal{C}$. To achieve \eqref{eq::coverage-guarantee}, we follow the idea of SCP \cite{DBLP:journals/ml/Vovk13, 10.5555/1062391} to construct $\mathcal{C}(x)$. The next section introduces the main idea of SCP.

\subsection{Split Conformal Prediction} \label{subsec::scp}

\textbf{Nonconformity score.} In SCP, we consider a supervised learning problem that involves predicting the label $y \in \mathcal{Y}$ of the input $x \in \mathcal{X}$. We assume that we have a predictive model $s: \mathcal{X} \times \mathcal{Y} \xrightarrow{} \mathbb{R}$, which outputs the nonconformity score $s(x, y)$. The nonconformity score function $s(\cdot, \cdot)$ is usually trained with a set of training data. $s(x,y) < s(x,y^\prime)$ means that for the input $x$, $y$ is more likely than $y^\prime$ to be the label. Some examples of nonconformity scores are as follows: for a probabilistic model $p(y|x)$, we can take the negative log-likelihood as the score, $s(x,y) = - \text{log} (p(y|x))$; for a regression model $h: \mathcal{X} \xrightarrow{} \mathcal{Y}$, a typical choice is $s(x,y) = | h(x) - y |$; for a multi-class classifier $h: \mathcal{X} \xrightarrow{} \Delta^{K-1}$, where $\Delta^{K-1}$ is the $K-1$ dimensional simplex in $\mathbb{R}^K$, we can take $s(x,y) = 1-h(x)_y$.

In SCP, we assume that the examples $\{ (X_i, Y_i) \}_{i=1}^{n+1} \subseteq \mathcal{X} \times \mathcal{Y}$ are exchangeable (\Cref{def::exchangeability}). For a predefined significance level $\alpha \in (0,1)$, the goal is to provide a valid confidence set $\widehat{\mathcal{C}}(X_{n+1})$. CP methods \cite{DBLP:journals/jmlr/ShaferV08} take advantage of the exchangeability of the data and the properties of the quantile function to make such a construction possible.
\begin{definition}[{\citep[Exchangeability]{DBLP:journals/jmlr/ShaferV08}}] \label{def::exchangeability}
    The random variables $Z_1, \dots, Z_n$ are exchangeable if for every permutation $\tau$ for integers $1, \dots, n$, the variables $W_1, \dots, W_n$, where $W_i = Z_{\tau(i)}$, have the same joint probability distribution as $Z_1, \dots, Z_n$.
\end{definition}

Let $V_i = s(X_i,Y_i)$ for $i \in [n+1]$ be the nonconformity scores corresponding to the examples $\{ (X_i, Y_i) \}_{i=1}^{n+1}$, where $s(\cdot, \cdot)$ is independent of $\{ (X_i, Y_i) \}_{i=1}^{n+1}$. The independence between $s(\cdot, \cdot)$ and $\{ (X_i, Y_i) \}_{i=1}^{n+1}$ is useful since in this case we can prove that the scores $\{ V_i\}_{i=1}^{n+1}$ are exchangeable. The exchangeability of $\{ V_i\}_{i=1}^{n+1}$ comes from the exchangeability of $\{ (X_i, Y_i) \}_{i=1}^{n+1}$ and the independence between the $s(\cdot, \cdot)$ and $\{ (X_i, Y_i) \}_{i=1}^{n+1}$. Next, define $\text{rank}(V_i)$ as the rank of $V_i$ among $\{ V_i\}_{i=1}^{n+1}$ for any $i\in[n+1]$ (in ascending order; we assume that ties are broken randomly). By the exchangeability of $\{ V_i\}_{i=1}^{n+1}$, $\text{rank}(V_i)$ is uniform on $[n+1]$, which is used to prove the validity of SCP in \Cref{lma::scp-coverage}. We use $\widehat{P}\left( \{ V_i\}_{i=1}^n \right)$ to denote the empirical distribution determined by the examples $V_1, \dots, V_n$. Let
\begin{equation} \label{set::scp}
    \widehat{\mathcal{C}}_n(x)\! \coloneqq \!\!\left\{ y\! \in\! \mathcal{Y} \Big| s(x,y) \!\le\! \mathcal{Q}\!\left(\!\! \frac{n+1}{n} (1\!-\!\alpha); \!\widehat{P}\!\left(\{ V_i \}_{i=1}^n \right)\!\right)\! \right\},
\end{equation}

we then have the following marginal coverage guarantee.

\begin{lemma} [The validity of SCP] \label{lma::scp-coverage}
    Assume that examples $\{ (X_i, Y_i) \}_{i=1}^{n+1}$ are exchangeable. For any nonconformity score $s(\cdot, \cdot)$ and any $\alpha \in (0,1)$, the prediction set defined in \Cref{set::scp} satisfies:
    \begin{equation}  \label{ieq::scp-coverage}
        \mathbb{P}\left( Y_{n+1} \in \widehat{\mathcal{C}}_n(X_{n+1}) \right) \ge 1 - \alpha,
    \end{equation}
    where the probability is over the randomness of $\{ (X_i, Y_i) \}_{i=1}^{n+1}$.
\end{lemma}

In the OOD generalization setting, we also want to obtain a valid set predictor that is valid for any $T \in \mathcal{T}$. In light of this, some natural questions arise:
\begin{quote}
    \emph{Does the set predictor defined in \Cref{set::scp} remain valid when the unseen target domain is different from the source domains? If not, can we construct a new set predictor that is valid in the OOD generalization setting?}
\end{quote}

Unfortunately, the answer to the first question is \textbf{negative}.  Theoretically, as shown in Appendix A.1, the proof of \Cref{lma::scp-coverage} is highly dependent on the exchangeability of the examples $\{ (X_i, Y_i) \}_{i=1}^{n+1}$, which is easily violated if there is any distributional shift between the distribution of $\{ (X_i, Y_i) \}_{i=1}^{n}$ and the distribution of $(X_{n+1},Y_{n+1})$. This means that in the OOD setting, the proof technique of \Cref{lma::scp-coverage} cannot be applied. Empirically, in Section \ref{sec::motivating-experiment}, we provide a toy example to show that the set predictor $\widehat{\mathcal{C}}_n(x)$ is no longer valid in the OOD setting.

In Section \ref{sec::ood-scp}, we give an \textbf{affirmative} answer to the second question. We first construct a new set predictor based on the $f$-divergence between the target domain and the convex hull of the source domains, then provide marginal coverage guarantees for the constructed predictor.

\section{SCP Fails in the OOD Setting} \label{sec::motivating-experiment}
In this section, we construct a toy example to show that for the OOD confidence set prediction problem, SCP is no longer valid, even under a slight distributional shift.

For simplicity, we consider a single-domain case. Specifically, we consider the regression problem and set $\mathcal{X} = \mathbb{R}^l$, $\mathcal{Y} = \mathbb{R}$. We define the source domain $S$ as follows: given a linear predictor $L(x) = \left< w^\star, x \right> + b^\star$ where $w^\star \in \mathbb{R}^l$ and $b^\star \in \mathbb{R}$. The marginal distribution of $X$ and the conditional distribution of $Y$ given $X$ are defined as:
\begin{equation*}
    X \sim \mathcal{N}(\mu_s, \sigma_{s,x}^2 I_l), \ \ \ \  Y|X=x \sim \mathcal{N}(L(x), \sigma_{s,y}^2),
\end{equation*}
where $\mu_s \in \mathbb{R}^l$ is the mean vector of $X$, $\sigma_{s,x}, \sigma_{s,y}$ are positive scalars, and $I_l \in \mathbb{R}^{l \times l}$ is an identity matrix. Similarly, for the target domain $T$, we define:
\begin{equation*}
    X \sim \mathcal{N}(\mu_t, \sigma_{t,x}^2 I_l), \ \ \ \  Y|X=x \sim \mathcal{N}(L(x), \sigma_{t,y}^2),
\end{equation*}
where $\mu_t \in \mathbb{R}^l$ is the mean vector of $X$ and $\sigma_{t,x}, \sigma_{t,y}$ are positive scalars. For simplicity, we set $\mu_s = \mu_t, \sigma_{s,x} = \sigma_{t,x}$ and $\sigma_{s,y} \ne \sigma_{t,y}$. We sample $m_\text{train}$ training examples from $S$ to train a linear predictor $\hat{L}(x) = \left< \hat{w}, x \right> + \hat{b}$, where $\hat{w} \in \mathbb{R}^l$ and $\hat{b} \in \mathbb{R}$. We then define the nonconformity score as $s(x,y) = | \hat{L}(x) - y |$. We sample $n$ examples from $S$ to construct the prediction set $\widehat{\mathcal{C}}_n(x)$ in \Cref{set::scp} and sample $m_\text{test}$ examples from $T$ to form the test data. We run $1000$ times with different random seeds. The results for the coverage (left) and length (right) of the prediction set are presented in box plot form in \Cref{fig::toy-example}. Here, the coverage is the ratio between the number of test examples such that $y_i \in \widehat{\mathcal{C}}_n(x_i)$ and the size of the test set. The red lines stand for the desired marginal coverages. Since the boxes are below the red coverage lines, we conclude that SCP fails to provide a prediction set with desired coverage when there exists a distributional shift between the source domain and the target domain.

\begin{figure}
  \centering
  \includegraphics[width=0.47\textwidth]{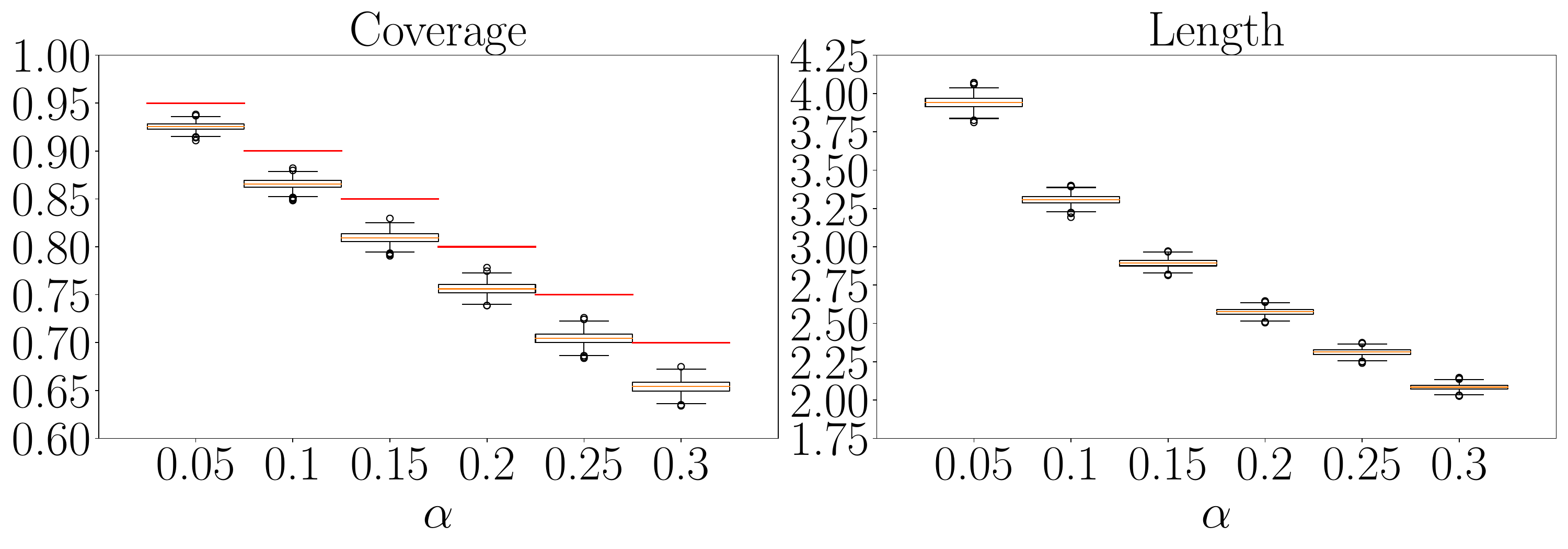}
  \caption{The box plots for the results of the $1000$ runs. We show the results for $\alpha=\{ 0.05, 0.1, 0.15, 0.2, 0.25, 0.3 \}$ and the horizontal axis represents the value of $\alpha$. The left plot shows the results for the coverage of the prediction sets. The red lines are the marginal coverage guarantees that we wish to achieve. The right plot shows the results for the length of the prediction sets.} \label{fig::toy-example}
\end{figure}

\section{Corrected SCP for OOD Data} \label{sec::ood-scp}
In this section, we consider correcting SCP for OOD data. We first consider the case in which we have access to the population distributions of the scores from the source domains. We then consider the case in which we only have access to the empirical distributions and correct the prediction set to obtain a marginal coverage guarantee in \Cref{ieq::scp-coverage}.

\subsection{Target Distribution Set and Confidence Sets}
It is obvious that obtaining a marginal coverage guarantee for an arbitrary target distribution is impossible unless we set $\widehat{\mathcal{C}}(x) = \mathcal{Y}$ for all $x \in \mathcal{X}$, which is a trivial confidence set predictor and does not provide any useful information. In this paper, we consider the case in which the $f$-divergence \cite{fdivergence} between the target domain and the convex hull of the source domains does not exceed a predefined value. The well-known KL divergence and TV distance are both special cases of $f$-divergence.

As Section \ref{sec::motivating-experiment} shows, when the target domain differs from the source domain, the marginal coverage does not hold for the predictor \eqref{set::scp}. Below, we construct new prediction sets where the marginal coverage \eqref{ieq::scp-coverage} holds.

We define a set of distributions $\mathcal{T} \subseteq \{ Q | Q \text{ is a distribution on } \mathcal{X} \times \mathcal{Y} \}$. For each $T \in \mathcal{T}$, the distribution of the score for the data from $T$ is defined as the push forward distribution $s\#T$, where $(s\#T)(A) = T(s^{-1}(A))$ for any measurable set $A \subseteq \mathbb{R}$. We define the distribution set of the scores as $\mathcal{P} \coloneqq \{ s\#T: T \in \mathcal{T} \}$. For a given $\alpha \in (0,1)$, our goal is to choose a threshold $t \in \mathbb{R}$ such that the confidence set $\widetilde{\mathcal{C}}(x) \coloneqq \{ y \in \mathcal{Y} | s(x,y) \le t \}$ satisfies \eqref{ieq::scp-coverage} when $(X_{n+1},Y_{n+1})$ is drawn from any target domain $T \in \mathcal{T}$. The following lemma provides a proper choice of $t$.

\begin{lemma} \label{lma::wc-scp-coverage}
    For any unknown target distribution $T \in \mathcal{T}$, assume that $(X_{n+1}, Y_{n+1})$ is drawn from $T$. If we set $t \ge \underset{P \in \mathcal{P}}{\max}\ \mathcal{Q}(1-\alpha;P)$, then:
    \begin{equation} \label{ieq-wc-scp-coverage}
        \mathbb{P}\left( Y_{n+1} \in \widetilde{\mathcal{C}}(X_{n+1}) \right) \ge 1 - \alpha.
    \end{equation}
\end{lemma}

For a given set $\mathcal{P}$ of distributions for the score, \Cref{lma::wc-scp-coverage} reduces the problem of finding a valid confidence set predictor to the following optimization problem:
\begin{equation} \label{problem::find-max-quantile}
    \max \ \mathcal{Q}(1-\alpha;P) \ \text{s.t.} \ P \in \mathcal{P}.
\end{equation}

Next, we formulate the set $\mathcal{T}$ through the lens of $f$-divergence.

\begin{definition}[$f$-divergence] \label{def::f-divergence}
Let $f: \mathbb{R} \xrightarrow{} \mathbb{R}$ be a closed convex function satisfying $f(1) = 0$ and $f(t) = + \infty$ for $t < 0$. Let $P,Q$ be two probability distributions such that $P \ll Q$ ($P$ is absolutely continuous with respect to $Q$). The $f$-divergence between $P$ and $Q$ can then be defined as follows:
\begin{equation*}
    D_f(P\Vert Q) \coloneqq \int f\left( \frac{dP}{dQ} \right) d Q,
\end{equation*}
where $\frac{dP}{dQ}$ is the \textbf{Radon-Nikodym derivative} \cite{billingsley2008probability}.
\end{definition}

\begin{remark} \label{rmk::def-f-divergence}
For a given function $f$ that satisfies the conditions in \Cref{def::f-divergence}, define $f_0(t) \coloneqq f(t) - f^\prime(1)(t-1)$. We then obtain that, for any $P\ll Q$:
\begin{equation*}
    \begin{aligned}
        &D_{f_0}(P\Vert Q) = \int f_0\left( \frac{dP}{dQ} \right) d Q = D_f(P\Vert Q).
    \end{aligned}
\end{equation*}
By the convexity of $f$, it can be easily observed that $f_0(t) \ge 0$ for all $t \in \mathbb{R}$. Moreover, $\inf_t f_0(t) = f_0(1)=0$ and $f_0^\prime(1) = 0$. Since $f_0$ produces the same $f$-divergence as $f$, without loss of generality, we can assume that $f^\prime(1)=f(1)=0$ and $f\ge 0$.
\end{remark}

Equipped with the $f$-divergence, we can now define our target distribution set $\mathcal{T}$ for a given threshold $\rho > 0$:
\begin{equation*} 
    \mathcal{T}_{f,\rho}\!(S_1\!, \!\cdots\!,\! S_d)\!\! \coloneqq\!\! \left\{ \!T | \exists Q \!\in\! \mathcal{CH}\!(S_1\!,\! \cdots\!,\! S_d)\ \text{s.t.} \ D_f(T\Vert Q) \!\!\le\!\! \rho\!\right\}.
\end{equation*}
We omit $S_1, \dots, S_d$ and use $\mathcal{T}_{f,\rho}$ for simplicity. The corresponding distribution set for the scores is then:
\begin{equation} \label{eq::score-set-corr}
    \mathcal{P} \coloneqq \{ s\# T | T \in \mathcal{T}_{f,\rho} \}.
\end{equation}

However, it is hard to obtain the precise relationship between $\mathcal{P}$ and the distributions $s\# S_1, \dots, s\# S_d$, which makes it difficult to analyze $\mathcal{P}$. We instead consider the following distribution set of scores:
\begin{equation} \label{eq::score-set}
    \begin{aligned}
        \mathcal{P}_{f,\rho} &\coloneqq \left\{ S \text{ is a distribution on } \mathbb{R} |  \right.\\
        &\exists S_0 \in \mathcal{CH}(s\# S_1, \cdots, s\# S_d)\ \left.\text{s.t.} \ D_f(S\Vert S_0) \le \rho\right\}.
    \end{aligned}
\end{equation}
The following lemma reveals the relationship between $\mathcal{P}$ and $\mathcal{P}_{f,\rho}$.

\begin{lemma} \label{lma::relation-2Ps}
    Let $\mathcal{P}$, $\mathcal{P}_{f,\rho}$ be defined as in \eqref{eq::score-set-corr}, \eqref{eq::score-set} respectively. Then, $\mathcal{P} \subseteq \mathcal{P}_{f,\rho}$.
\end{lemma}

\begin{remark} \label{rmk::lma-relation-2Ps}
    According to \Cref{lma::relation-2Ps}, $\underset{P \in \mathcal{P}_{f,\rho}}{\sup} \mathcal{Q}(1-\alpha;P) \ge \underset{P \in \mathcal{P}}{\sup} \mathcal{Q}(1-\alpha;P)$. \Cref{lma::wc-scp-coverage} accordingly tells us that if we set $t = \underset{P \in \mathcal{P}_{f,\rho}}{\sup} \mathcal{Q}(1-\alpha;P)$, then for $(X_{n+1},Y_{n+1})$ drawn from any target distribution $T \in \mathcal{T}_{f,\rho}$, we have $\mathbb{P}\left( Y_{n+1} \in \widetilde{\mathcal{C}}(X_{n+1}) \right) \ge 1 - \alpha$. Our goal is now to solve Problem \eqref{problem::find-max-quantile} for the set $\mathcal{P}_{f,\rho}$.
\end{remark}

According to \Cref{rmk::lma-relation-2Ps}, we define the worst-case quantile function for the distribution set $\mathcal{P}_{f,\rho}$ as $\widetilde{\mathcal{Q}}(\alpha; \mathcal{P}_{f,\rho}) \coloneqq \underset{P \in \mathcal{P}_{f,\rho}}{\sup} \mathcal{Q}(\alpha; P)$. \Cref{rmk::lma-relation-2Ps} tells us that taking $t = \widetilde{\mathcal{Q}}(1-\alpha; \mathcal{P}_{f,\rho})$ produces a valid confidence set $\widetilde{\mathcal{C}}$. The next theorem allows us to express the worst-case quantile function in terms of the standard quantile function, which helps us to calculate the worst-case quantile efficiently.

\begin{theorem} \label{thm::wc-quantile-to-standard-quantile}
Let $F_1, \dots, F_d$ be the c.d.f.'s of the distributions $s\# S_1, \dots, s\# S_d$. Define the function $g_{f,\rho}: [0,1] \xrightarrow{} [0,1]$ as
\begin{equation*}
    g_{f,\rho}(\beta)\! \coloneqq\! \inf\! \left\{\! z \!\in\! [0,1] \bigg| \beta f\!\left( \frac{z}{\beta} \right) \!+\! (1\!-\!\beta) f\!\left( \frac{1-z}{1-\beta} \right) \!\le\! \rho \right\}
\end{equation*}
and define the inverse of $g_{f,\rho}$ as $g_{f,\rho}^{-1}(\tau) = \sup \{ \beta \in [0,1] \big| g_{f,\rho}(\beta) \le \tau \}$. Let $F_{\min}(x) \coloneqq \underset{1 \le i \le d}{\min} F_i(x)$ be a c.d.f., the following holds for all $\alpha \in (0,1)$:
\begin{equation*}
    \widetilde{\mathcal{Q}}(\alpha; \mathcal{P}_{f,\rho}) = \mathcal{Q}(g_{f,\rho}^{-1}(\alpha); F_{\min}).
\end{equation*}
\end{theorem}

\subsection{Marginal Coverage Guarantee for Empirical Source Distributions} \label{subsec::coverage-for-empirical-distributions}
In the previous section, we presented marginal coverage guarantees when we have access to the population distributions of the scores for source domains. However, in practice, it is difficult or even impossible to access these population distributions. In this section, we provide marginal coverage guarantees even when we only have access to the empirical distributions, which is useful in practice.

For any $i \in [d]$, assume we have $m_i$ i.i.d. examples $\left\{ V_{ij} = s(X_{ij}, Y_{ij}) \right\}_{j=1}^{m_i}$ from the source distribution $S_i$. Further, suppose that $\hat{F}_i$ is the empirical c.d.f. corresponding to $F_i$, which is defined as $\hat{F}_i(x) = \frac{1}{m_i} \sum_{j=1}^{m_i} \mathbb{I}_{ (-\infty, x]}(V_{ij})$. Define $\hat{F}_{\min}(x) = \underset{1 \le i \le d}{\min} \hat{F}_i(x)$. We first provide an error bound when we estimate $F_{\min}$ with $\hat{F}_{\min}$.

\begin{proposition} \label{prop::Fmin-empirical-approx}
    Let $F_1, \dots, F_d$ be c.d.f.'s on $\mathbb{R}$, define $F_{\min}(x) = \underset{1 \le i \le d}{\min} F_i(x)$. Suppose $\hat{F}_1, \dots, \hat{F}_d$ are the empirical c.d.f.'s corresponding to $F_1, \dots, F_d$, defined with $m_1, \dots, m_d$ examples, respectively. Define $\hat{F}_{\min}(x) = \underset{1 \le i \le d}{\min} \hat{F}_i(x)$. Then, for any $\epsilon > 0$,
    \begin{equation*}
        \mathbb{P}\left( \underset{x \in \mathbb{R}}{\sup} \left| F_{\min}(x) - \hat{F}_{\min}(x) \right| > \epsilon \right) \le 2 \sum_{i=1}^d e^{-2m_i \epsilon^2},
    \end{equation*}
    where the probability is over the randomness of the examples that define the empirical c.d.f.'s.
\end{proposition}

The above \Cref{prop::Fmin-empirical-approx} allows us to quantify the error caused by replacing the population distributions with the empirical distributions, which leads to the following marginal coverage guarantee for the prediction set $\widetilde{\mathcal{C}}$ that we have defined before.

\begin{theorem} [Marginal coverage guarantee for the empirical estimations] \label{thm::coverage-for-finite-approx}
    Assume $V_{n+1} = s(X_{n+1}, Y_{n+1})\sim P \in \mathcal{P}_{f,\rho}$ is independent of $\{ V_{ij}\}_{i,j=1}^{d, m_i}$ where $\{ V_{ij}\}_{j=1}^{m_i} \overset{i.i.d.}{\sim} s\#S_i$ for $i \in [d]$. Suppose $\rho^\star = \underset{P_0 \in \mathcal{CH}_s}{\inf} D_f(P\Vert P_0) \le \rho$ where $\mathcal{CH}_s = \mathcal{CH}(s\# S_1, \cdots, s\# S_d)$. Let $\hat{F}_{\min}$ be defined as in \Cref{prop::Fmin-empirical-approx} and let $\hat{S}_1, \dots, \hat{S}_d$ be the empirical distributions of $S_1, \dots, S_d$ respectively. If we set $t = \widetilde{\mathcal{Q}}(1-\alpha; \hat{\mathcal{P}}_{f,\rho}) = \mathcal{Q}( g_{f,\rho}^{-1}(1-\alpha) ; \hat{F}_{\min} )$, then for any $\epsilon > 0$, we obtain the following marginal coverage guarantee for $\widetilde{\mathcal{C}}$:
    \begin{equation*}
        \begin{aligned}
            \mathbb{P}\!\left(\!\! Y_{n+1} \!\!\in\! \widetilde{\mathcal{C}} (X_{n+1}) \!\right) &\!\!\ge\!\! \left( \!\!1 \!\!-\!\! 2 \!\sum_{i=1}^d \!e^{-2 m_i \epsilon^2} \!\!\right)\!\! g_{f,\rho^\star}\!\!\left( \!g_{f,\rho}^{-1}(1\!\!-\!\!\alpha) \!-\! \epsilon\! \right),
        \end{aligned}
    \end{equation*}    
    where the randomness is over the choice of the source examples and $(X_{n+1}, Y_{n+1})$ and
    \begin{equation*}
        \hat{\mathcal{P}}_{f,\rho} \!\coloneqq\!\!\left\{ \!S \Big| \exists S_0 \!\!\in\! \mathcal{CH}(s\# \hat{S}_1, \cdots, s\# \hat{S}_d)\ \text{s.t.} \ D_f\!(S\Vert S_0) \!\!\le\!\! \rho\!\right\}.
    \end{equation*}
\end{theorem}

By Lemma 14 in the Appendix, $g_{f,\rho}(\beta)$ is non-increasing in $\rho$ and non-decreasing in $\beta$, so $g_{f,\rho^\star}(g_{f,\rho}^{-1}(1\!-\!\alpha) \!-\! \epsilon) \!\ge\! g_{f,\rho}(g_{f,\rho}^{-1}(1\!-\!\alpha) \!-\! \epsilon)$. In practice, we do not know $\rho^\star$, so we use $g_{f,\rho}(g_{f,\rho}^{-1}(1\!-\!\alpha) \!-\! \epsilon)$ instead. Since $g_{f,\rho}(g_{f,\rho}^{-1}(1\!-\!\alpha) \!-\! \epsilon) \!\le\! g_{f,\rho}(g_{f,\rho}^{-1}(1\!-\!\alpha)) \!=\! 1-\alpha$, we get guaranteed coverage $\left(1\!-\!2\sum_{i=1}^d e^{-2m_i \epsilon^2}\right) g_{f,\rho}(g_{f,\rho}^{-1}(1\!-\!\alpha) \!-\! \epsilon) \!\le\! 1\!-\!\alpha$. To achieve a marginal coverage with the level of at least $1-\alpha$, we need to correct the output set by replacing $\alpha$ with some $\alpha^\prime < \alpha$ when running our confidence set predictor. The following corollary tells us how to choose $\alpha^\prime$ to correct the prediction set.

\begin{corollary} [Correct the prediction set to get a $(1-\alpha)$ marginal coverage] \label{cor::correct-coverage-guarantee}
    Let $(X_{n+1}, Y_{n+1})$, $\hat{F}_{\min}$, $\hat{\mathcal{P}}_{f,\rho}$ be defined as in \Cref{thm::coverage-for-finite-approx}. For arbitrary $\epsilon > 0$, if we set $t = \widetilde{\mathcal{Q}}(1-\alpha^\prime; \hat{\mathcal{P}}_{f,\rho}) = \mathcal{Q}\left( g_{f,\rho}^{-1}(1-\alpha^\prime) ; \hat{F}_{\min} \right)$, where
    \begin{equation*}
        \alpha^\prime = 1 - g_{f,\rho}\left( \epsilon + g_{f,\rho}^{-1}\left( \frac{1-\alpha}{1 - 2 \sum_{i=1}^d e^{-2 m_i \epsilon^2}} \right) \right),
    \end{equation*}
    then we obtain the following marginal coverage guarantee:
    \begin{equation*}
         \mathbb{P}\left( Y_{n+1} \in \widetilde{\mathcal{C}} (X_{n+1}) \right) \ge 1 - \alpha.
    \end{equation*}
\end{corollary}

\begin{remark} \label{rmk::solve-g-inverse}
    \Cref{cor::correct-coverage-guarantee} tells us that we can take $t = \mathcal{Q}\left( g_{f,\rho}^{-1}(1-\alpha^\prime) ; \hat{F}_{\min} \right) = \mathcal{Q}\left( \epsilon + g_{f,\rho}^{-1}\left( \frac{1-\alpha}{1 - 2 \sum_{i=1}^d e^{-2 m_i \epsilon^2}} \right); \hat{F}_{\min} \right)$ to get a marginal coverage guarantee with confidence level $1-\alpha$. When $f(\cdot), s(\cdot, \cdot)$ are chosen and the numbers of examples that are used to estimate the source distributions, i.e., $m_1, \dots, m_d$, are given, we solve the following optimization problem to find a desired $t$.
    \begin{equation*} 
        \begin{aligned}
            \underset{0 < \epsilon \le 1}{\min}&\ \  \mathcal{Q}\left( \epsilon + g_{f,\rho}^{-1}\left( \frac{1-\alpha}{1 - 2 \sum_{i=1}^d e^{-2 m_i \epsilon^2}} \right); \hat{F}_{\min}\right),\\
            \text{  s.t.  } & \epsilon + g_{f,\rho}^{-1}\left( \frac{1-\alpha}{1 - 2 \sum_{i=1}^d e^{-2 m_i \epsilon^2}} \right) \le 1.
        \end{aligned}
    \end{equation*}
    Since the quantile function $\mathcal{Q}\left(\cdot; \hat{F}_{\min}\right)$ is non-decreasing, let $h(\epsilon) = \epsilon + g_{f,\rho}^{-1}\left( \frac{1-\alpha}{1 - 2 \sum_{i=1}^d e^{-2 m_i \epsilon^2}} \right)$, we solve the following problem instead:
    \begin{equation*} 
        \min\ \  h(\epsilon)  \text{    s.t.  } 0 < \epsilon \le 1, h(\epsilon) \le 1.
    \end{equation*}
    For some choices of $f$, the functions $g_{f,\rho}$ and $g_{f,\rho}^{-1}$ have closed forms (please refer to the examples in Section \ref{subsec::examples}). For general $f$ that we do not have a closed form of $g_{f,\rho}^{-1}$, the following lemma tells us that we can use a binary search algorithm to efficiently compute the value of $g_{f,\rho}^{-1}(\tau)$ for a given $\tau$.
\end{remark}
\begin{lemma} [\citep{DBLP:journals/corr/abs-2008-04267}, The form of $g_{f,\rho}^{-1}$ that can be efficiently solved] \label{lma::g-inverse-form}
    Let $g_{f,\rho}, g_{f,\rho}^{-1}$ be defined as in \Cref{thm::wc-quantile-to-standard-quantile}. Then, for any $\tau \in [0,1]$, we have:
    \begin{equation*}
        g_{f,\rho}^{-1}(\tau) \!=\! \sup \!\left\{ \beta \!\in\! [\tau, 1]\left| \beta f\!\left( \frac{\tau}{\beta} \right) \!+\! (1\!-\!\beta) f\!\left( \frac{1\!-\!\tau}{1\!-\!\beta} \right) \!\le\! \rho \right. \!\right\}.
    \end{equation*}
\end{lemma}

\subsection{Examples} \label{subsec::examples}
In this section, we present some examples of calculating $g_{f,\rho}$ and $g_{f,\rho}^{-1}$ for some important $f$-divergences.

\begin{example} [$\chi^2$-divergence] \label{example::chi-square-div}
    Let $f(t) = (t-1)^2$; then, $D_f(P\Vert Q) = \mathbb{E}_Q \left[ \left(\frac{dP}{dQ} -1 \right)^2 \right] = \mathbb{E}_Q \left[ \left(\frac{dP}{dQ} \right)^2 -1 \right]$ is the $\chi^2$-divergence. In this case, we have $g_{f,\rho}(\beta) = \left( \beta - \sqrt{\rho \beta(1-\beta)} \right)_+$, where $(x)_+ = \max\{0,x \}$. $g_{f,\rho}^{-1}(\tau)$ is the solution of the following optimization problem:
    \begin{equation*}
        \max \beta \text{   s.t. } \left\{
                                        \begin{array}{l}
                                        \frac{\rho}{\rho + 1} \le \beta \le 1 \\
                                        \beta - \sqrt{\rho \beta(1-\beta)} \le \tau
                                        \end{array} \right..
    \end{equation*}
\end{example}

\begin{example} [Total variation distance, \citep{DBLP:journals/corr/abs-2008-04267}] \label{example::TV}
    Let $f(t) = \frac{1}{2} |t-1|$; then, $D_f(P\Vert Q) = \mathbb{E}_Q \left[ \frac{1}{2} \left|\frac{dP}{dQ} -1 \right| \right]$ is the total variation distance. In this case, we can provide analytic forms for $g_{f, \rho}$ and $g_{f, \rho}^{-1}$:
    \begin{equation*}
        g_{f, \rho}(\beta) = (\beta - \rho)_+, \ \ \ \ g_{f, \rho}^{-1}(\tau) = \min \{ \tau + \rho, 1 \}.
    \end{equation*}
\end{example}

\begin{example} [Kullback-Leibler divergence] \label{example::KL}
    Let $f(t) = t \log t$; then, $D_f(P\Vert Q) = \mathbb{E}_Q \left[ \frac{dP}{dQ} \log\left(\frac{dP}{dQ}\right) \right]$ is the Kullback-Leibler (KL) divergence \cite{KL-divergence}. Unfortunately, we cannot provide the analytic forms of $g_{f, \rho}$ and $g_{f, \rho}^{-1}$ for KL-divergence. Fortunately, according to \Cref{thm::wc-quantile-to-standard-quantile} and \Cref{rmk::solve-g-inverse}, we can compute the values $g_{f, \rho}(\beta)$ and $g_{f, \rho}^{-1}(\tau)$ by solving a one-dimensional convex optimization problem, which can be solved efficiently using binary search.
\end{example}

\begin{figure}[!ht]
  \centering
  \includegraphics[width=0.47\textwidth]{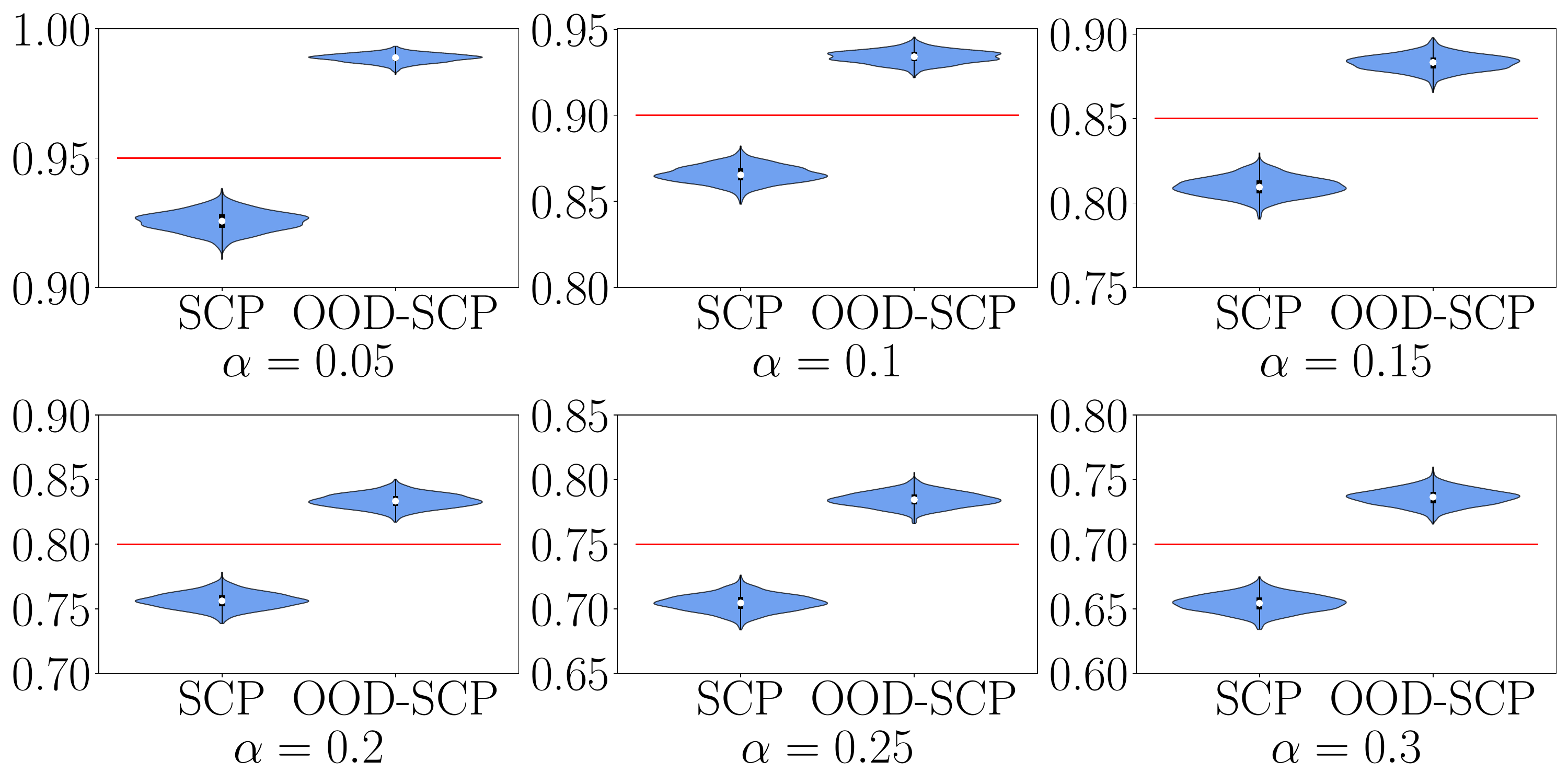}
  \caption{The violin plots for the coverage of the $1000$ runs under the same data generation settings as in Section \ref{sec::motivating-experiment}. We show results for $\alpha=\{0.05, 0.1, 0.15, 0.2, 0.25, 0.3\}$. Here, the red lines are the marginal coverage guarantees that we wish to achieve. The white point represents the median, while the two endpoints of the thick line are the $0.25$ quantile and the $0.75$ quantile.} \label{fig::exp-source1-coverage}
\end{figure}

\begin{figure}[!ht]
  \centering
  \includegraphics[width=0.47\textwidth]{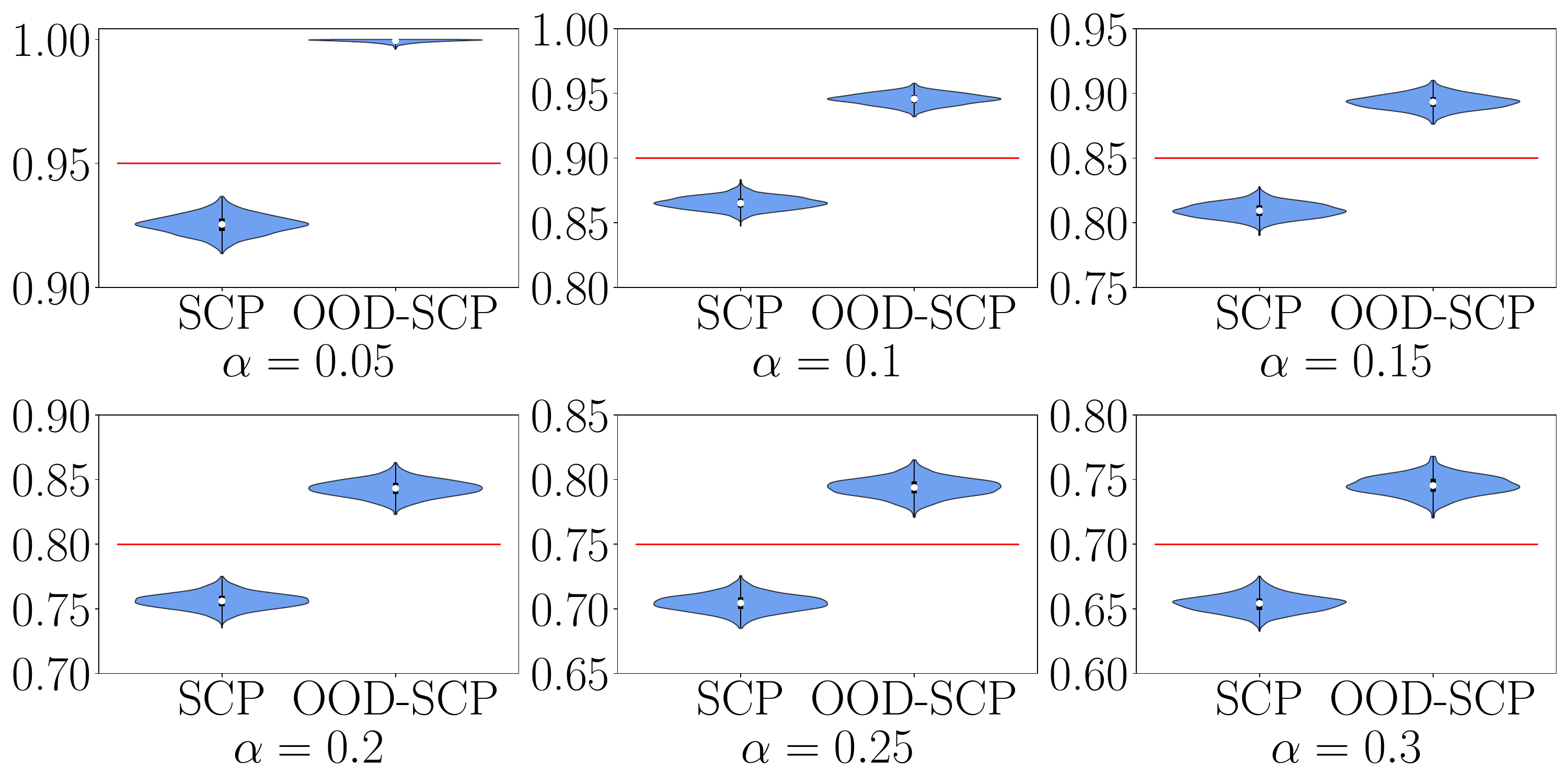}
  \caption{The violin plots for the coverage of the $1000$ runs for the multi-source OOD confidence set prediction task. We show results for $\alpha=\{0.05, 0.1, 0.15, 0.2, 0.25, 0.3\}$. Here, the red lines are the marginal coverage guarantees that we wish to achieve. The white point represents the median, while the two endpoints of the thick line are the $0.25$ quantile and the $0.75$ quantile.} \label{fig::exp-source2-coverage}
\end{figure}

\section{Experiments} \label{sec::experiments}
In this section, we use simulated data to verify our theory and the validity of our constructed confidence set predictor (referred to as \textbf{OOD-SCP} in the remainder of this paper). We consider two cases: first, we verify the validity of OOD-SCP using the same settings as in Section \ref{sec::motivating-experiment}; then, we construct a multi-source OOD confidence set prediction task and show that OOD-SCP is valid for this task.

According to \Cref{fig::exp-source1-coverage}, unlike standard SCP, for all values of $\alpha$, the violin for OOD-SCP is above the desired coverage line, which shows that OOD-SCP is empirically valid.

We next consider a multi-source OOD confidence set prediction task. Similar to Section \ref{sec::motivating-experiment}, we consider the regression problem and set $\mathcal{X} = \mathbb{R}^l, \mathcal{Y} = \mathbb{R}$. Define the oracle linear predictor $L: \mathcal{X} \xrightarrow{} \mathcal{Y}$ as $L(x) = \left< w^\star, x \right> + b^\star$, where $w^\star \in \mathbb{R}^l$ and $b^\star \in \mathbb{R}$. We define the marginal distribution of $X$ for the source domains $S_1$ and $S_2$ as 
\begin{equation*}
    S_{1X} = \mathcal{N}(\mu_1, \sigma_{s,x}^2 I_l), \ \ \ \ S_{2X} = \mathcal{N}(\mu_2, \sigma_{s,x}^2 I_l)
\end{equation*}
respectively, where $\mu_1, \mu_2 \in \mathbb{R}^l$ are the mean vectors, $\sigma_{s,x} > 0$ is a scalar, and $I_l \in \mathbb{R}^{l \times l}$ is the identity matrix with dimension $l \times l$. We define $Y|X=x \sim \mathcal{N}(L(x), \sigma_{s,y}^2)$ for both $S_1$ and $S_2$. For the target domain $T$, we define the marginal distribution of $X$ as $T_X = \frac{S_{1X} + S_{2X}}{2}$ and the conditional distribution of $Y$ given $X$ as $Y|X=x \sim \mathcal{N}(L(x), \sigma_{t,y}^2)$. Here, $\sigma_{s,y}, \sigma_{t,y} > 0$ are the standard deviations and $\sigma_{s,y} \ne \sigma_{t,y}$.

Similar to Section \ref{sec::motivating-experiment}, we sample $\frac{m_\text{train}}{2}$ examples from $S_1$ and $\frac{m_\text{train}}{2}$ examples from $S_2$ to train a linear predictor $\hat{L}(x) = \left< \hat{w}, x \right> + \hat{b}$, where $\hat{w} \in \mathbb{R}^l$ and $\hat{b} \in \mathbb{R}$. We then define the nonconformity score as $s(x,y) = | \hat{L}(x) - y |$. We sample $\frac{n}{2}$ examples from $S_1$, and $\frac{n}{2}$ examples from $S_2$ to construct the prediction set $\widetilde{\mathcal{C}}(x)$ and sample $m_\text{test}$ examples from $T$ to form the test data.

\Cref{fig::exp-source2-coverage} shows the results for the multi-source OOD confidence set prediction task. From the figure, we can see that the violins for the standard SCP are under the desired coverage lines, which means that the standard SCP is invalid in this case. By contrast, the violins for OOD-SCP are above the desired coverage lines, indicating that OOD-SCP is valid, which validates \Cref{cor::correct-coverage-guarantee}.

The reason we do not do experiments on real datasets is that do not know how to set the value of $\rho$ for the existing OOD datasets. Our main claim is that \textbf{when the target domain satisfies $T\in\mathcal{T} _{f,\rho}$, the coverage of our method is guaranteed}. However, we claim \textbf{it is acceptable}. In many fields, we face the same problem.

In adversarial robustness \citep{DBLP:journals/corr/SzegedyZSBEGF13}, the theories (for example \citep{DBLP:conf/colt/MontasserHS19}) provide an upper bound of the test adversarial robustness $\underset{(x,y)\sim D}{\mathbb{P}}[\exists \Vert\delta\Vert\le\epsilon: h(x+\delta) \ne y ]$, where $h$ is a classifier. The results just tell us that we have the guarantee for the test accuracy \textbf{if the test perturbation $\delta$ satisfies $\Vert\delta\Vert\le\epsilon$}. However, what if $\Vert\delta\Vert > \epsilon$? It is out of the scope of their theories.

For distributional robustness optimization (DRO), the theories \citep{DBLP:conf/nips/LeeR18} prove that \textbf{if the test distribution is in a Wasserstein ball with radius $r$}, then the test risk can be upper bounded. Formally, $\underset{D\in W(r)}{\max} \underset{(x,y)\sim D}{\mathbb{P}}[h(x) \ne y ]$ is upper bounded, where $W(r)$ is a Wasserstein ball with radius. They do not know how to set $r$ to make $W(r)$ contain the test distribution either, however, this does not overshadow their contribution to the DRO community. In other words, the issues of $\rho$ do not overshadow our contribution to the OOD community.

\section{Discussions} \label{sec::discussion}
Our work is an extension of \citep{DBLP:journals/corr/abs-2008-04267} to the multi-domain case. In this section, we discuss the differences between our work and \citep{DBLP:journals/corr/abs-2008-04267}.

\subsection{The Necessity of Our Extension}
In the multi-source setting, to make use of all the source domains, a trivial method is to regard the mixture of the source domains,
as a domain $S = \sum_{i=1}^d \lambda_i S_i$ and use the method in \citep{DBLP:journals/corr/abs-2008-04267}. However, there are two issues:
\begin{itemize}
    \item Given the empirical data from $S_1, \dots, S_d$, we don’t know the exact values of $\lambda_1, \!\dots, \!\lambda_d$ for the mixed domain, so we don't know the set $\bar{\mathcal{P}} = \left\{ s\#T | D_f(T\Vert S) \le \rho \right\}$ for a given $\rho$. So we don't know the set that we are giving a coverage guarantee for.
    \item We may be not able to provide a coverage guarantee for data from one of the source domains. Take KL-divergence as an example, then drawing from $S$ can be regarded as first drawing an index $I$ from $\lambda$ and then drawing an example from $S_I$. $S_i$ can be seen as drawn from the same process with $\lambda=e_i$, where $e_i = (0,\dots, 0, 1, 0, \dots, 0) \in \mathbb{R}^d$ with only the $i$-th element being $1$. By the chain rule, $KL(S_i\Vert S) = KL(e_i\Vert \pmb{\lambda}) + \mathbb{E}_{j \sim e_i} KL(S_j \Vert S_j) = \log (1/\lambda_i)$. If $\rho < \max_i \log (1/\lambda_i)$, then there exists a $S_i$ s.t. $KL(S_i\Vert S) > \rho$, i.e., \textbf{we can not even get a coverage guarantee for the source domain $S_i$, which is unacceptable!} The problem gets worse if the source domain number $d$ becomes larger since $\max_i \log (1/\lambda_i) \ge \log d$. However, in our generalization, there is no such problem even if we choose $\rho=0$, which means \textbf{the method in \citep{DBLP:journals/corr/abs-2008-04267} is not compatible with the multi-source setting, so our extension is necessary}.
\end{itemize}

\subsection{The Difference in Proof Skills}
In fact, our Theorem 6 is an extension of \citep{DBLP:journals/corr/abs-2008-04267} to the OOD setting and mainly depends on Lemmas 17 and 18. Lemma 17 helps us reduce multi-input $g_{f,\rho}$ to a single-input case. The main idea of the proof of Lemma 18 comes from the argument in \citep{DBLP:journals/corr/abs-2008-04267}, however, \textbf{the extension is non-trivial}. In Lemma 18, let $h(z,\beta) \!=\! \beta f(z/\beta) + (1\!-\!\beta) f((1\!-\!z)/(1\!-\!\beta))$ we use the multi-input $g_{f,\rho}(\beta_1, \!\cdots\!, \beta_d)\!=\!\inf\{z\!\in\![0,1] | \inf_{\lambda \in \Delta^{d-1}} h(z,\sum_{i=1}^d \!\lambda_i \beta_i) \!\le\! \rho \}$, which involves taking infimum w.r.t. $\lambda$ and is much more complicated than the single-input case in \citep{DBLP:journals/corr/abs-2008-04267}. We construct a set $\mathcal{P}_{f,\rho}^*$ that is more complicated than that in \citep{DBLP:journals/corr/abs-2008-04267} and the proof is more difficult. Moreover, due to multiple inputs and the $\inf_{\lambda \in \Delta^{d-1}}$ operator, we need to consider 4 cases according to whether each $F_i(t)$ is $0$ or $1$.

Our Theorem 8 and its corresponding Corollary 9 are novel and quite different from Corollaries 2.1 and 2.2 in \citep{DBLP:journals/corr/abs-2008-04267}. The common point is that they all consider finite sample approximation. The proof of Corollary 2.1 in \citep{DBLP:journals/corr/abs-2008-04267} relies on the exchangeability of the source examples, however, in the OOD setting, examples are drawn from different source domains and are not exchangeable. So the analysis techniques in \citep{DBLP:journals/corr/abs-2008-04267} can not be applied in our case. To fill this gap, we use the decomposition technique and concentration inequalities.

\section{Conclusion} \label{sec::conclusion}
We study the confidence set prediction problem in the OOD generalization setting. We first empirically show that SCP is not valid in the OOD generalization setting. We then develop a method for forming valid confident prediction sets in the OOD setting and theoretically prove the validity of our proposed method. Finally, we conduct experiments on simulated data to empirically verify both the correctness of our theory and the validity of our proposed method.

\section*{Acknowledgements}
This work is supported by the National Key R\&D Program of China under Grant 2023YFC3604702, the National Natural Science Foundation of China under Grant 61976161, the Fundamental Research Funds for the Central Universities under Grant 2042022rc0016.

\bibliography{Formatting-Instructions-LaTeX-2024}

\appendix
\setcounter{table}{0}
\setcounter{figure}{0}
\setcounter{equation}{0}

\renewcommand{\thefigure}{B.\arabic{figure}}
\renewcommand{\thetable}{B.\arabic{table}}
\renewcommand{\theequation}{A.\arabic{equation}}

\onecolumn

\section{Proofs} \label{sec::proofs}
\subsection{Proof of \Cref{lma::scp-coverage}} \label{prf::lma-scp-coverage}

\begin{customlemma} {\ref{lma::scp-coverage}}
    Assume that examples $\{ (X_i, Y_i) \}_{i=1}^{n+1}$ are exchangeable. For any nonconformity score $s$ and any $\alpha \in (0,1)$, the prediction set defined in \Cref{set::scp} satisfies:
    \begin{equation*}
        \mathbb{P}\left( Y_{n+1} \in \widehat{\mathcal{C}}_n(X_{n+1}) \right) \ge 1 - \alpha,
    \end{equation*}
    where the probability is over the randomness of $\{ (X_i, Y_i) \}_{i=1}^{n+1}$.
\end{customlemma}

\begin{proof} [Proof of \Cref{lma::scp-coverage}]
    Let $V_i = s(X_i,Y_i)$, then the following holds:
    \begin{equation*}
        \begin{aligned}
            \mathbb{P}\left( Y_{n+1} \in \widehat{\mathcal{C}}_n(X_{n+1}) \right) &\overset{\sroman{1}}{=} \mathbb{P}\left( V_{n+1} \le \mathcal{Q}\left( \frac{n+1}{n} (1-\alpha); \widehat{P}\left( \{ V_i \}_{i=1}^n \right) \right) \right)\\
            &\overset{\sroman{2}}{=} \mathbb{P}\left( \mathbb{P}(S\le V_{n+1}) \le \frac{n+1}{n}(1-\alpha) \right) \\
            &\ge \mathbb{P}\left( \text{rank}(V_{n+1}) \le \Ceil{(n+1)(1-\alpha)} \right) \\
            &\ge 1-\alpha,
        \end{aligned}
    \end{equation*}
    where $\sroman{1}$ is from the definition of $\widehat{\mathcal{C}}_n$ and $\sroman{2}$ is a result of the definition of the quantile function $\mathcal{Q}$, the inner probability is over the randomness of $S$ and the outer probability is over the randomness of $\{ (X_i, Y_i) \}_{i=1}^{n+1}$.
\end{proof}

\subsection{Proof of \Cref{lma::wc-scp-coverage}} \label{prf::lma-wc-scp-coverage}
\begin{customlemma} {\ref{lma::wc-scp-coverage}}
    For any unknown target distribution $T \in \mathcal{T}$, assume that $(X_{n+1}, Y_{n+1})$ is drawn from $T$. If we set $t \ge \underset{P \in \mathcal{P}}{\max}\ \mathcal{Q}(1-\alpha;P)$, then:
    \begin{equation*}
        \mathbb{P}\left( Y_{n+1} \in \widetilde{\mathcal{C}}(X_{n+1}) \right) \ge 1 - \alpha.
    \end{equation*}
\end{customlemma}

\begin{proof} [Proof of \Cref{lma::wc-scp-coverage}]
    \begin{equation*}
        \begin{aligned}
            \mathbb{P}\left( Y_{n+1} \in \widetilde{\mathcal{C}}(X_{n+1}) \right) &= \mathbb{P}\left( s(X_{n+1},Y_{n+1}) \le t \right) = \mathbb{P}\left( V_{n+1} \le t \right) \\
            &\overset{\sroman{1}}{\ge} \mathbb{P}\left( V_{n+1} \le \underset{P \in \mathcal{P}_0}{\max}\ \mathcal{Q}(1-\alpha;P) \right) \\
            &\overset{\sroman{2}}{\ge} \mathbb{P}\left( V_{n+1} \le \mathcal{Q}(1-\alpha; s\#T) \right) \\
            &\overset{\sroman{3}}{\ge} 1 - \alpha,
        \end{aligned}
    \end{equation*}
    where $\sroman{1}$ is from the fact $t \ge \underset{P \in \mathcal{P}_0}{\max}\ \mathcal{Q}(1-\alpha;P)$; $\sroman{2}$ is because $\underset{P \in \mathcal{P}_0}{\max}\ \mathcal{Q}(1-\alpha; P) \ge \mathcal{Q}(1-\alpha; s\#T)$ and $\sroman{3}$ is a result of the definition of the quantile function.
\end{proof}

\subsection{Proof of \Cref{lma::relation-2Ps}} \label{prf::lma-relation-2Ps}
\begin{customlemma}{\ref{lma::relation-2Ps}}
    Let $\mathcal{P}$, $\mathcal{P}_{f,\rho}$ be defined as in \eqref{eq::score-set-corr}, \eqref{eq::score-set} respectively. Then, $\mathcal{P} \subseteq \mathcal{P}_{f,\rho}$.
\end{customlemma}

The proof of \Cref{lma::relation-2Ps} relies on the data processing inequality for $f$-divergence.

\begin{lemma} [Data Processing Inequality for $f$-divergence]\label{lma::f-divergence-data-processing-ieq}
    Let $P_X, Q_X$ be distributions on measurable space $(\mathcal{X}, \mathcal{F})$ and $P_{Y|X}$ be a transition kernel from $(\mathcal{X}, \mathcal{F})$ to $(\mathcal{Y}, \mathcal{G})$. Let $P_Y, Q_Y$ be distributions on measurable space $(\mathcal{Y}, \mathcal{G})$ and be the transformation of $P_X, Q_X$ pushed through $P_{Y|X}$, i.e., $P_X(B) = \int_\mathcal{X} P_{Y|X}(B|x) dP_X(x)$. Then, for any f-divergence, we have that:
    \begin{equation*}
        D_f(P_X\Vert Q_X) \ge D_f(P_Y \Vert Q_Y).
    \end{equation*}
\end{lemma}

To prove \Cref{lma::f-divergence-data-processing-ieq}, we need some properties about the $f$-divergences.
\begin{lemma} [Properties of $f$-divergences] \label{lma::f-divergence-properties}
For any distributions $P,Q$ on $\mathcal{Z}$ dominated by a common measure $\lambda$, i.e. $P,Q \ll \lambda$ (such a $\lambda$ always exists since we can take $\lambda = \frac{1}{2} P + \frac{1}{2} Q$), we have:

\begin{itemize}
    \item \textbf{Non-Negativity:} $D_f(P\Vert Q) \ge 0$. Furthermore, if $f$ is strictly convex around $1$, then the equality is obtained if and only if $P=Q$.
    \item \textbf{Convexity:} The mapping $(P,Q) \xrightarrow{} D_f(P\Vert Q)$ is jointly convex. Consequently, $P \xrightarrow{} D_f(P\Vert Q)$ is convex for fixed $Q$, and $Q \xrightarrow{} D_f(P\Vert Q)$ is convex for fixed $P$.
    \item \textbf{Joint vs. Marginal:} If $P_{XY} = P_X P_{Y|X}$ and $Q_{XY} = Q_X P_{Y|X}$, then $D_f\left( P_{XY} \Vert Q_{XY} \right) = D_f\left( P_X \Vert Q_X \right)$.
\end{itemize}
\end{lemma}

\begin{proof} [Proof of \Cref{lma::f-divergence-properties}]
    To prove the \textbf{non-negativity}, by the definition of $f$-divergence we have:
    \begin{equation*}
        D_f(P\Vert Q) = \mathbb{E}_Q\left[ f\left( \frac{dP}{dQ} \right) \right] \overset{\sroman{1}}{\ge} f\left( \mathbb{E}_Q\left[\frac{dP}{dQ} \right] \right) = f(1) = 0,
    \end{equation*}
    where $\sroman{1}$ follows from Jensen's inequality. Next, we prove that if $f$ is strictly convex around $1$, then the equality is obtained if and only if $P=Q$. It is obvious that if $P = Q$, then $f\left( \frac{dP}{dQ} \right) = f(1) = 0$, so $D_f(P\Vert Q) = 0$. For the other direction, as stated in \Cref{rmk::def-f-divergence}, we assume $f \ge 0$, so $D_f(P\Vert Q) = 0$ implies that $f\left( \frac{dP}{dQ} \right) = 0$ almost everywhere. Since $f$ in strongly convex around $1$ and $f(1)=0$, then $f\left( \frac{dP}{dQ} \right) = 0$ implies $\frac{dP}{dQ} = 1$, which means that $P = Q$.

    To prove the \textbf{convexity}, let $\mathbb{R}_* = [0, +\infty)$, for $f: \mathbb{R}_* \xrightarrow{} \mathbb{R}$, we define $f_\text{perspective}: \mathbb{R}_* \times \mathbb{R}_* \xrightarrow{} \mathbb{R}$ as:
    \begin{equation*}
        f_\text{perspective}(x,y) = y f\left( \frac{x}{y} \right), \ \ \ \text{dom}(f_\text{perspective}) = \left\{ (x,y) \in \mathbb{R}_* \times \mathbb{R}_*| \frac{x}{y} \in \text{dom}(f)\right\}.
    \end{equation*}
    Since $f$ is convex, according to \citep[Section 3.2.6, Page 89]{boyd2004convex}, $f_\text{perspective}$ is also convex. Denote $D_f^*(P, Q) = D_f(P \Vert Q)$, let $\mathcal{M}(\mathcal{X})$ be the domain of $D_f(\cdot\Vert \cdot)$, for $(P_1,Q_1),(P_2,Q_2) \in \mathcal{M}(\mathcal{X})$, suppose that $P_1,Q_1,P_2,Q_2 \ll \lambda$ (such $\lambda$ exists, e.g., we can set $\lambda = \frac{P_1 + P_2 + Q_1 + Q_2}{4}$). Then, for any $\alpha \in [0,1]$ we have:
    \begin{equation*}
        \begin{aligned}
            D_f^*( \alpha(P_1, Q_1) &+ (1-\alpha) (P_2, Q_2) ) \!=\!\! \int_\mathcal{X} f\!\left( \!\frac{\alpha\frac{dP_1}{d\lambda} + (1-\alpha)\frac{dP_2}{d\lambda}}{\alpha\frac{dQ_1}{d\lambda} + (1-\alpha)\frac{dQ_2}{d\lambda}} \!\right) \left( \alpha\frac{dQ_1}{d\lambda} + (1-\alpha)\frac{dQ_2}{d\lambda}\right) d\lambda \\
            &=\int_\mathcal{X} f_\text{perspective}\left(\alpha\frac{dP_1}{d\lambda} + (1-\alpha)\frac{dP_2}{d\lambda}, \alpha\frac{dQ_1}{d\lambda} + (1-\alpha)\frac{dQ_2}{d\lambda}\right) d\lambda \\
            &\overset{\sroman{1}}{\le} \int_\mathcal{X} \alpha f_\text{perspective}\left(\frac{dP_1}{d\lambda}, \frac{dQ_1}{d\lambda}\right) + (1-\alpha) f_\text{perspective}\left(\frac{dP_2}{d\lambda}, \frac{dQ_2}{d\lambda}\right) d\lambda \\
            &= \alpha \int_\mathcal{X} \frac{dQ_1}{d\lambda} f\left( \frac{\frac{dP_1}{d\lambda}}{\frac{dQ_1}{d\lambda}} \right) d\lambda + (1-\alpha) \int_\mathcal{X} \frac{dQ_2}{d\lambda} f\left( \frac{\frac{dP_2}{d\lambda}}{\frac{dQ_2}{d\lambda}} \right) d\lambda \\
            &= \alpha D_f^*(P_1, Q_1) + (1-\alpha) D_f^*(P_2, Q_2),
        \end{aligned}
    \end{equation*}
    where $\sroman{1}$ is from the convexity of $f_\text{perspective}$.

    To prove \textbf{the relationship between $f$-divergence of joint distributions and marginal distributions}, we have:
    \begin{equation*}
        \begin{aligned}
            D_f(P_{XY} \Vert Q_{XY}) &= \underset{Q_{XY}}{\mathbb{E}} \left[ f\left( \frac{dP_{XY}}{dQ_{XY}} \right) \right] = \underset{Q_{XY}}{\mathbb{E}} \left[ f\left( \frac{dP_X P_{X|Y}}{dQ_X P_{X|Y}} \right) \right] \\
            &= \underset{Q_{XY}}{\mathbb{E}} \left[ f\left( \frac{dP_X}{dQ_X} \right) \right] = \underset{Q_X}{\mathbb{E}} \left[ f\left( \frac{dP_X}{dQ_X} \right) \right] = D_f(P_X \Vert Q_X).
        \end{aligned}
    \end{equation*}
\end{proof}

\begin{proof} [Proof of \Cref{lma::f-divergence-data-processing-ieq}]
    By \Cref{lma::f-divergence-properties}, if $P_{XY} = P_X P_{Y|X}$ and $Q_{XY} = Q_X P_{Y|X}$, then we have:
    \begin{equation*}
        D_f(P_X \Vert Q_X) = D_f( P_{XY} \Vert Q_{XY}) = \underset{Q_{XY}}{\mathbb{E}} \left[ f\left( \frac{dP_{XY}}{dQ_{XY}} \right) \right].
    \end{equation*}
    Using the law of total expectation, we get:
    \begin{equation*}
        \underset{Q_{XY}}{\mathbb{E}} \left[ f\left( \frac{dP_{XY}}{dQ_{XY}} \right) \right] = \underset{Q_Y}{\mathbb{E}} \left[ \underset{Q_{X|Y}}{\mathbb{E}} \left[ f\left( \frac{dP_{XY}}{dQ_{XY}} \right) \bigg| Y \right] \right].
    \end{equation*}
    Since $f$ is convex, Jensen's inequality tells us that:
    \begin{equation*}
        \underset{Q_Y}{\mathbb{E}} \left[ \underset{Q_{X|Y}}{\mathbb{E}} \left[ f\left( \frac{dP_{XY}}{dQ_{XY}} \right) \bigg| Y \right] \right] \ge \underset{Q_Y}{\mathbb{E}} \left[ f\left( \underset{Q_{X|Y}}{\mathbb{E}}\left[ \frac{dP_{XY}}{dQ_{XY}} \bigg| Y \right] \right) \right].
    \end{equation*}
    Then we consider the $\underset{Q_{X|Y}}{\mathbb{E}}\left[ \frac{dP_{XY}}{dQ_{XY}} \bigg| Y \right]$ term, we have:
    \begin{equation*}
        \underset{Q_{X|Y}}{\mathbb{E}}\left[ \frac{dP_{XY}}{dQ_{XY}} \bigg| Y \right] = \int_\mathcal{X} \frac{dP_{XY}}{dQ_{XY}} dQ_{X|Y} = \int_\mathcal{X} \frac{dP_Y dP_{X|Y}}{dQ_Y dQ_{X|Y}} dQ_{X|Y} \overset{\sroman{1}}{=} \int_\mathcal{X} \frac{dP_Y}{dQ_Y} dP_{X|Y} = \frac{dP_Y}{dQ_Y},
    \end{equation*}
    where $\sroman{1}$ is from the fact that $Q_{X|Y} = P_{X|Y}$. Then we have:
    \begin{equation*}
        D_f(P_X\Vert Q_X) = \underset{Q_{XY}}{\mathbb{E}} \left[ f\left( \frac{dP_{XY}}{dQ_{XY}} \right) \right] \ge \underset{Q_Y}{\mathbb{E}} \left[ f\left( \frac{dP_Y}{dQ_Y} \right) \right] = D_f(P_Y \Vert Q_Y).
    \end{equation*}
\end{proof}

\begin{proof} [Proof of \Cref{lma::relation-2Ps}]
    For any $S \in \mathcal{P}$, by the definition of $\mathcal{P}$, we know that: there exists $T \in \mathcal{T}_{f,\rho}$ such that $S = s\# T$.
    
    By the definition of $\mathcal{T}_{f,\rho}$, we know that there exists $Q \in \mathcal{CH}(S_1, \cdots, S_d)\ \text{s.t.} \ D_f(T\Vert Q) \le \rho$. By \Cref{lma::f-divergence-data-processing-ieq}, we have: $D_f(s\# T \Vert s\# Q) \le D_f(T \Vert Q) \le \rho$. Moreover, $Q \in \mathcal{CH}(S_1, \cdots, S_d)$ implies that $s\# Q \in \mathcal{CH}(s\# S_1, \cdots, s\# S_d)$, by the definition of $\mathcal{P}_{f,\rho}$, we know that $S \in \mathcal{P}_{f,\rho}$. So $\mathcal{P} \subseteq \mathcal{P}_{f,\rho}$.
\end{proof}

\subsection{Proof of \Cref{thm::wc-quantile-to-standard-quantile}} \label{prf::thm-wc-quantile-to-standard-quantile}
\begin{customthm}{\ref{thm::wc-quantile-to-standard-quantile}}
Let $F_1, \dots, F_d$ be the c.d.f.'s of the distributions $s\# S_1, \dots, s\# S_d$. Define the function $g_{f,\rho}: [0,1] \xrightarrow{} [0,1]$ as:
\begin{equation*}
    g_{f,\rho}(\beta) \coloneqq \inf \left\{ z \in [0,1] \bigg| \beta f\left( \frac{z}{\beta} \right) + (1-\beta) f\left( \frac{1-z}{1-\beta} \right) \le \rho \right\},
\end{equation*}
then for inverse of $g_{f,\rho}$:
\begin{equation*}
    g_{f,\rho}^{-1}(\tau) = \sup \{ \beta \in [0,1] \big| g_{f,\rho}(\beta) \le \tau \},
\end{equation*}
the following holds for all $\alpha \in (0,1)$:
\begin{equation*}
    \widetilde{\mathcal{Q}}(\alpha; \mathcal{P}_{f,\rho}) = \mathcal{Q}(g_{f,\rho}^{-1}(\alpha); F_{\min}),
\end{equation*}
where $F_{\min}(x) \coloneqq \underset{1 \le i \le d}{\min} F_i(x)$ is a c.d.f.
\end{customthm}

Firstly, we prove that if $F_1, \dots, F_d$ are c.d.f.'s, then $F_{\min}(x) = \underset{1 \le i \le d}{\min} F_i(x)$ is a c.d.f.
\begin{lemma} \label{lma::Fmin-cdf}
    Suppose that $F_1, \dots, F_d$ are c.d.f.'s, let $F_{\min}(x) = \underset{1 \le i \le d}{\min} F_i(x)$, then $F_{\min}$ is a c.d.f.
\end{lemma}
\begin{proof} [Proof of \Cref{lma::Fmin-cdf}]
    We need to prove that $F_{\min}$ satisfies the following properties:\
    \begin{itemize}
        \item[\textbf{P1}] $F_{\min}$ is non-decreasing.
        \item[\textbf{P2}] $0 \le F_{\min}(x) \le 1$, $\underset{x \to -\infty}{\lim} F_{\min}(x) = 0$ and $\underset{x \to +\infty}{\lim} F_{\min}(x) = 1$.
        \item[\textbf{P3}] $F_{\min}$ is right continuous.
    \end{itemize}
    \textbf{Proof of P1:} Since $F_1, \dots, F_d$ are all c.d.f.'s, $\forall x \le y$, we have $F_i(x) \le F_i(y)$ for all $i \in [d]$.

    Suppose $j = \underset{1 \le i \le d}{\arg \min} F_i(y)$, then $F_j(y) = \underset{1 \le i \le d}{\min} F_i(y)$. Since $F_j(x) \le F_j(y)$, we have that:
    \begin{equation*}
        F_{\min}(x) = \underset{1 \le i \le d}{\min} F_i(x) \le F_j(x) \le F_j(y) = \underset{1 \le i \le d}{\min} F_i(y) = F_{\min}(y),
    \end{equation*}
    which means that the function $F_{\min}$ is non-decreasing.

    \textbf{Proof of P2:} Since $F_1, \dots, F_d$ are all c.d.f.'s, then $0 \le F_i(x) \le 1$, $\underset{x \to -\infty}{\lim} F_i(x) = 0$ for all $i \in [d]$. It is obvious that $0 \le F_{\min}(x) \le 1$. 
    
    Since $\underset{x \to -\infty}{\lim} F_i(x) = 0$, $\forall \epsilon > 0, \forall i \in [d]$, there exists $\delta_i$ such that $\forall x \le \delta_i$, we have $| F_i(x) | < \epsilon$, i.e., $0 \le F_i(x) < \epsilon$. Let $\delta = \underset{1 \le i \le d}{\min} \delta_i$, then, if $x \le \delta$, $0 \le F_i(x) < \epsilon$ holds for all $i\in [d]$, so $F_{\min}(x) = \underset{1 \le i \le d}{\min} F_i(x) < \epsilon$, which means that $\underset{x \to -\infty}{\lim} F_{\min}(x) = 0$. We can prove that $\underset{x \to +\infty}{\lim} F_{\min}(x) = 1$ similarly.

    \textbf{Proof of P3:} Since $F_1, \dots, F_d$ are all c.d.f.'s, they are all right continuous. Fix any $x_0$, then by the definition of right continuity, we know that: $\forall \epsilon > 0$, $\forall i \in [d]$, there exists $\delta_i > 0$ such that, for any $x$ satisfying $0 < x - x_0 < \delta_i$, we have:
    \begin{equation*}
        |F_i(x) - F_i(x_0)| < \epsilon,\  \text{i.e.}, F_i(x_0) - \epsilon < F_i(x) < F_i(x_0) + \epsilon.
    \end{equation*}
    Let $\delta = \underset{1 \le i \le d}{\min} \delta_i$, set $j = \underset{1 \le i \le d}{\arg \min} F_i(x)$, $k = \underset{1 \le i \le d}{\arg \min} F_i(x_0)$, then:
    \begin{equation*}
        F_{\min}(x) =  F_j(x) = \underset{1 \le i \le d}{\min} F_i(x), \ F_{\min}(x_0) =  F_k(x_0) = \underset{1 \le i \le d}{\min} F_i(x_0).
    \end{equation*}
    So, for all $0 < x-x_0 < \delta$, we have:
    \begin{equation*}
        F_{\min}(x_0) - \epsilon = \underset{1 \le i \le d}{\min} F_i(x_0) -\epsilon \le F_j(x_0) - \epsilon < F_j(x) = F_{\min}(x).
    \end{equation*}
    Similarly, we have:
    \begin{equation*}
        F_{\min}(x) = \underset{1 \le i \le d}{\min} F_i(x) \le F_k(x) < F_k(x_0) + \epsilon = F_{\min}(x_0) + \epsilon.
    \end{equation*}
    So we have that $|F_{\min}(x) - F_{\min}(x_0)| < \epsilon$
\end{proof}

We now show some useful properties for $g_{f,\rho}$, which are useful in our proof.
\begin{lemma} [Lemma A.1, page 33 of \citep{DBLP:journals/corr/abs-2008-04267}] \label{lma::properties-of-g}
    Let $f$ be a function that satisfies the conditions in \Cref{def::f-divergence}, then the function $g_{f,\rho}$ defined in \Cref{thm::wc-quantile-to-standard-quantile} satisfies the following properties:
    \begin{enumerate}[(1)]
        \item The function $(\beta, \rho) \xrightarrow{} g_{f,\rho}(\beta)$ is a convex function.
        \item The function $(\beta, \rho) \xrightarrow{} g_{f,\rho}(\beta)$ is continuous for $\beta \in [0,1]$ and $\rho \in (0,\infty)$.
        \item $g_{f,\rho}$ is non-increasing in $\rho$ and non-decreasing in $\beta$. Moreover, for all $\rho > 0$, there exists $\beta_0(\rho) \coloneqq \sup \{ \beta \in (0,1) | g_{f,\rho}(\beta) = 0\} $, $g_{f,\rho}$ is strictly increasing for $\beta > \beta_0(\rho)$.
        \item For $\beta \in [0,1]$ and $\rho >0$, $g_{f,\rho}(\beta) \le \beta$. For $\rho > 0$, equality holds for $\beta = 0$, strict inequality holds for $\beta \in (0,1)$ and $\rho > 0$, and $g_{f, \rho}(1)=1$ if and only if $f^\prime(\infty) = \infty$.
        \item Let $g_{f,\rho}^{-1}(t) = \sup \{\beta| g_{f,\rho}(\beta) \le t \}$ as in the statement of \Cref{thm::wc-quantile-to-standard-quantile}. Then for $\tau \in (0,1)$, $g_{f,\rho}(\beta) \ge \tau$ if and only if $g_{f,\rho}^{-1}(\tau) \le \beta$.
    \end{enumerate}
\end{lemma}

We then find the corresponding c.d.f. of the worst-case quantile function $\widetilde{\mathcal{Q}}$ by the following Lemma.
\begin{lemma} \label{lma::worst-case-F}
    The c.d.f. that corresponds to the worst-case quantile function $\widetilde{\mathcal{Q}}$ (we call it worst-case c.d.f.) has the formulation that:
    \begin{equation} \label{eq::worst-case-F}
        \widetilde{F}(s; \mathcal{P}_{f,\rho}) = \underset{P \in \mathcal{P}_{f,\rho}}{\inf} P(S\le s)
    \end{equation}
\end{lemma}

\begin{proof} [Proof of \Cref{lma::worst-case-F}]
    By the definition of the quantile function, suppose the corresponding c.d.f. of $\widetilde{\mathcal{Q}}(\cdot; \mathcal{P}_{f,\rho})$ is $\widetilde{F}$, then we have the following:
    \begin{equation} \label{eq::worst-case-F-quantile-relation}
        \widetilde{\mathcal{Q}}(\beta; \mathcal{P}_{f,\rho}) = \inf \{ s \in \mathbb{R} | \widetilde{F}(s; \mathcal{P}_{f,\rho}) \ge \beta \}.
    \end{equation}
    According to the definition of $\widetilde{\mathcal{Q}}(\beta; \mathcal{P}_{f,\rho})$, for any $\beta \in [0,1]$, we have:
    \begin{equation} \label{eq::worst-case-quantile-infer}
        \begin{aligned}
            \widetilde{\mathcal{Q}}(\beta; \mathcal{P}_{f,\rho}) &= \underset{P \in \mathcal{P}_{f,\rho}}{\sup} \mathcal{Q}(\beta; P) = \underset{P \in \mathcal{P}_{f,\rho}}{\sup} \inf \{ t \in \mathbb{R} | P(S\le t)\ge \beta \}\\
            &\overset{\sroman{1}}{=} \underset{P \in \mathcal{P}_{f,\rho}}{\sup} \{ t \in \mathbb{R} | P(S>t) \ge 1-\beta \} = \sup \left\{ t \in \mathbb{R} \Bigg| \underset{P \in \mathcal{P}_{f,\rho}}{\sup} P(S>t) \ge 1-\beta \right\} \\
            &= \inf \left\{ t \in \mathbb{R} \Bigg| \underset{P \in \mathcal{P}_{f,\rho}}{\sup} P(S>t) \le 1-\beta \right\} = \inf \left\{ t \in \mathbb{R} \Bigg| 1 - \underset{P \in \mathcal{P}_{f,\rho}}{\sup} P(S>t) \ge \beta \right\} \\
            &= \inf \left\{ t \in \mathbb{R} \Bigg| \underset{P \in \mathcal{P}_{f,\rho}}{\inf} \left(1 - P(S>t)\right) \ge \beta \right\} = \inf \left\{ t \in \mathbb{R} \Bigg| \underset{P \in \mathcal{P}_{f,\rho}}{\inf} P(S\le t) \ge \beta \right\}
        \end{aligned}
    \end{equation}
    where $\sroman{1}$ follows from the fact that for all $t \in \mathbb{R}$, $t < \mathcal{Q}(\beta; P)$ if and only if $P(S\le t)<\beta$, i.e., if and only if $P(S>t) \ge 1-\beta$. Compare \Cref{eq::worst-case-F-quantile-relation} with \Cref{eq::worst-case-quantile-infer}, we know that \Cref{eq::worst-case-F} holds.
\end{proof}

Since the distribution set $\mathcal{P}_{f,\rho}$ involves $d$ distributions $s\# S_1, s\# S_2, \dots, s\# S_d$, with a slight abuse of notation, we need to extend $g_{f,\rho}$ to be a function such that $g_{f,\rho}: [0,1]^+ \xrightarrow{} [0,1]$, where $[0,1]^+$ means that the input number is adaptive and can be $1,2,\dots$.
\begin{definition} \label{def::adatptive-g}
    We define the adaptive version of $g_{f,\rho}: [0,1]^+ \xrightarrow{} [0,1]$ as:
    \begin{equation*}
        g_{f,\rho}(\beta_1, \cdots, \beta_d) \coloneqq \inf \left\{ z \in [0,1] \Bigg| \underset{\sum_{i=1}^d \lambda_i = 1}{\underset{\lambda_1, \dots, \lambda_d \ge 0}{\inf}} h\left( z, \sum_{i=1}^d \lambda_i \beta_i \right) \le \rho \right\},
    \end{equation*}
    where
    \begin{equation*}
        h\left( z, \beta \right) = \beta f\left( \frac{z}{\beta} \right) + (1-\beta) f\left( \frac{1-z}{1-\beta} \right).
    \end{equation*}
\end{definition}

It's obvious that when the input number is $1$, the function $g_{f,\rho}$ in \Cref{def::adatptive-g} reduces to the $g_{f,\rho}$ in \Cref{thm::wc-quantile-to-standard-quantile}. The following Lemma reduces $g_{f,\rho}(\beta_1, \dots, \beta_d)$ to $g_{f,\rho}\left(\underset{1\le i \le d}{\min}\beta_i \right)$.
\begin{lemma} \label{lma::multi-g-reduce-to-single-g}
    For any $d \in \mathbb{N}_+$ and $\beta_1, \dots, \beta_d \in [0,1]$, we have that:
    \begin{equation*}
        g_{f,\rho}(\beta_1, \dots, \beta_d) = g_{f,\rho}\left(\underset{1\le i \le d}{\min}\beta_i \right).
    \end{equation*}
\end{lemma}

\begin{proof} [Proof of \Cref{lma::multi-g-reduce-to-single-g}]
    For all $0 \le a_1 \le a_2 \le 1$, we define:
    \begin{equation*}
        g_{f,\rho,*}(a_1,a_2) = \inf\left\{ z \in [0,1] \bigg| \underset{\beta \in [a_1,a_2]}{\inf} \beta f\left( \frac{z}{\beta} \right) + (1-\beta) f\left( \frac{1-z}{1-\beta} \right) \le \rho \right\}.
    \end{equation*}
    For fixed $d$ and $\beta_1, \dots, \beta_d \in [0,1]$, it is obvious that:
    \begin{equation*}
        \left\{ \sum_i^d \lambda_i \beta_i \bigg| \lambda_1, \dots, \lambda_d \ge 0, \sum_i^d \lambda_i = 1 \right\} = \left[ \underset{1\le i \le d}{\min} \beta_i, \underset{1\le i \le d}{\max} \beta_i \right],
    \end{equation*}
    which implies that $g_{f,\rho}(\beta_1, \dots, \beta_d) = g_{f,\rho,*}\left(\underset{1\le i \le d}{\min} \beta_i, \underset{1\le i \le d}{\max} \beta_i\right)$. Now, it is sufficient to prove that for all $0 \le a_1 \le a_2 \le 1$, we have:
    \begin{equation*}
        g_{f,\rho,*}(a_1,a_2) = g_{f,\rho}(a_1).
    \end{equation*}
    Recall that $h\left( z, \beta \right) = \beta f\left( \frac{z}{\beta} \right) + (1-\beta) f\left( \frac{1-z}{1-\beta} \right)$, we now analyze the properties of $h\left( z, \beta \right)$.

    Let $f_\text{perspective}(z,\beta) = \beta f \left( \frac{z}{\beta} \right)$, according to \citep[Section 3.2.6, Page 89]{boyd2004convex}, $f_\text{perspective}(z,\beta)$ is also convex. Let $A = \begin{pmatrix} -1 & 0 \\ 0 & -1 \end{pmatrix}$ and $b = (1, 1)^T$, then we know that $(1-z, 1-\beta)^T = A (z, \beta)^T + b$. According to \citep[Section 3.2.2, Page 79]{boyd2004convex}, the convexity is preserved when the input vector is composited with an affine mapping. So the function $(z,\beta) \xrightarrow{} f_\text{perspective}(1-z,1-\beta)$ is also convex. Then we know that $h(z,\beta) = f_\text{perspective}(z,\beta) + f_\text{perspective}(1-z,1-\beta)$ is convex.

    Recall that in \Cref{rmk::def-f-divergence}, we show that we can assume, without loss of generality, that $f^\prime(1)=0$ and $f\ge 0$. Now we take the partial derivative of $h$ with respect to $\beta$ and get:
    \begin{equation*}
        \begin{aligned}
            \frac{\partial h}{\partial \beta}\bigg|_{(z,\beta)} &= f\left( \frac{z}{\beta} \right) + \beta f^\prime \left( \frac{z}{\beta} \right) \left( -\frac{z}{\beta^2} \right) - f\left( \frac{1-z}{1-\beta} \right) + (1-\beta) f^\prime \left( \frac{1-z}{1-\beta} \right)\left( \frac{1-z}{(1-\beta)^2} \right) \\
            &= f\left( \frac{z}{\beta} \right) - \frac{z}{\beta} f^\prime \left( \frac{z}{\beta} \right) - f\left( \frac{1-z}{1-\beta} \right) + \frac{1-z}{1-\beta} f^\prime \left( \frac{1-z}{1-\beta} \right).
        \end{aligned}
    \end{equation*}
    Taking the partial derivative of $h$ with respect to $z$ shows:
    \begin{equation*}
        \frac{\partial h}{\partial z}\bigg|_{(z,\beta)} = \beta f^\prime\left( \frac{z}{\beta} \right) \frac{1}{\beta} + (1-\beta) f^\prime\left( \frac{1-z}{1-\beta} \right) \frac{-1}{1-\beta} = f^\prime\left( \frac{z}{\beta} \right) - f^\prime\left( \frac{1-z}{1-\beta} \right)
    \end{equation*}
    Fix $z \in [0,1]$, solving the equation $\frac{\partial h}{\partial \beta}\big|_{(z,\beta)} = 0$ gives $\beta = z$. Similarly, fix $\beta \in [0,1]$, solving the equation $\frac{\partial h}{\partial z}\big|_{(z,\beta)} = 0$ gives $z = \beta$. So we know that:
    \begin{equation*}
        \underset{\beta \in [a_1, a_2]}{\inf} h(z,\beta) = \left\{
                                                                \begin{array}{lcl}
                                                                h(z,z)=0 & & {z\in [a_1,a_2]}\\
                                                                h(z,a_1) & & {z < a_1}\\
                                                                h(z,a_2) & & {z>a_2}
                                                                \end{array} \right..
    \end{equation*}
    Then we discuss three situations:
    \begin{enumerate}[(1)]
        \item If $g_{f,\rho,*}(a_1,a_2)$ is attained when $z<a_1$, then $g_{f,\rho,*}(a_1,a_2) = \inf \{ t\in [0,a_1) | h(t,a_1) \le \rho \}$.
        \item If $g_{f,\rho,*}(a_1,a_2)$ is attained when $z\in [a_1,a_2]$, then $g_{f,\rho,*}(a_1,a_2) = \inf \{ t\in [a_1,a_2] | h(t,t) \le \rho \} = a_1$.
        \item If $g_{f,\rho,*}(a_1,a_2)$ is attained when $z>a_2$, then $g_{f,\rho,*}(a_1,a_2) = \inf \{ t\in (a_2,1) | h(t,a_2) \le \rho \} \ge a_1$.
    \end{enumerate}
    In conclusion, we know that $g_{f,\rho,*}(a_1,a_2) \le a_1$, so we have:
    \begin{equation*}
        g_{f,\rho,*}(a_1,a_2) = \inf \{ z \in [0,a_1] | h(z,a_1) \le \rho \} = \inf \{ z \in [0,1] | h(z,a_1) \le \rho \} = g_{f,\rho}(a_1),
    \end{equation*}
    which finishes the proof.
\end{proof}

\begin{lemma} \label{lma::worst-case-F-multi-g-relation}
    Let $F_1, \dots, F_d$ be the c.d.f.'s of the distributions $s\# S_1, \dots, s\# S_d$ and define $F_{\min}(x) = \underset{1\le i \le d}{\min} F_i(x)$. Then we have:
    \begin{equation*}
        \widetilde{F}(t; \mathcal{P}_{f,\rho}) = g_{f,\rho}\left( F_1(t), \cdots, F_d(t) \right) = g_{f,\rho}\left( F_{\min}(t)\right)
    \end{equation*}
\end{lemma}

\begin{proof} [Proof of \Cref{lma::worst-case-F-multi-g-relation}]
    The second equality is a direct result from \Cref{lma::multi-g-reduce-to-single-g} We prove the first equality in four cases.
    \begin{enumerate}[(1)]
        \item \textbf{When there exists $j \in [d]$ such that $F_j(t) = 0$}. Recall that $\widetilde{F}(t; \mathcal{P}_{f,\rho}) = \underset{P \in \mathcal{P}_{f,\rho}}{\inf} P(t\le t)$, we know that $0\le \widetilde{F}(t; \mathcal{P}_{f,\rho}) \le \underset{1 \le i \le d}{\min} F_i(t) = 0$, so $\widetilde{F}(t; \mathcal{P}_{f,\rho}) = 0$.
        
        Setting $z = 0$, since $F_j(t) = 0$, take $\lambda_j=1$ and $\lambda_i = 0$ for all $i\ne j$, then we get that:
        \begin{equation*}
            \begin{aligned}
                &\underset{\sum_{i=1}^d \lambda_i = 1}{\underset{\lambda_1, \dots, \lambda_d \ge 0}{\inf}} \left( \sum_{i=1}^d \lambda_i \beta_i \right)f\left( \frac{z}{\sum_{i=1}^d \lambda_i \beta_i} \right) + \left(1-\sum_{i=1}^d \lambda_i \beta_i\right) f\left( \frac{1-z}{1-\sum_{i=1}^d \lambda_i \beta_i} \right) \\
                &\le \beta_j f\left(\frac{z}{\beta_j} \right) + (1-\beta_j)f\left(\frac{1-z}{1-\beta_j} \right) = 0f\left( \frac{0}{0} \right) + f(1) = 0,
            \end{aligned}
        \end{equation*}
        which means that $g_{f,\rho}\left( F_1(t), \cdots, F_d(t) \right) = 0$. So, in this case, we have: $\widetilde{F}(t; \mathcal{P}_{f,\rho}) = g_{f,\rho}\left( F_1(t), \cdots, F_d(t) \right)$. Here, we can define $0f\left( \frac{0}{0} \right)$ by taking $\beta_j = z$ and let $\beta_j, z \to 0$, which means that $0f\left( \frac{0}{0} \right)=0$ here.
        \item \textbf{When $0 < F_i(t) < 1$ for all $i \in [n]$}, in this proof, we denote $\mathcal{CH}(s\# S_1, \cdots, s\# S_d)$ by $\mathcal{CH}_s$ for simplicity, note that the set
        \begin{equation*}
            \mathcal{P}_{f,\rho}^* \coloneqq \left\{ P \bigg| \exists P_0 \in \mathcal{CH}_s \text{ s.t. } \ D_f(P\Vert P_0) \le \rho, \frac{dP}{dP_0} \text{ is constant on } \{ S \le t \} \text{ and } \{ S > t \} \right\}
        \end{equation*}
        is a subset of $\mathcal{P}_{f,\rho}$, i.e., $\mathcal{P}_{f,\rho}^* \subseteq \mathcal{P}_{f,\rho}$. Now we consider what is meant by $D_f(P\Vert P_0) \le \rho, \frac{dP}{dP_0} \text{ is constant on } \{ S \le t \} \text{ and } \{ S > t \}$. Suppose that $\frac{dP}{dP_0} = p_1$ on $\{ S \le t \}$ and $\frac{dP}{dP_0} = p_2$ on $\{ S > t \}$, on the one hand:
        \begin{equation*}
            \int_{\{S\le t\}} \frac{dP}{dP_0} dP_0 = \int_{\{S\le t\}} dP = P(S \le t);
        \end{equation*}
        on the other hand:
        \begin{equation*}
            \int_{\{S\le t\}} \frac{dP}{dP_0} dP_0 = \int_{\{S\le t\}} p_1 dP_0 = p_1 \cdot P_0(S \le t).
        \end{equation*}
        The above two equations imply that $P(S \le t) = p_1 \cdot P_0(S \le t)$, i.e., $p_1 = \frac{P(S \le t)}{P_0(S \le t)}$. Similarly we can prove that $p_2 = \frac{P(S > t)}{P_0(S > t)}$. Then we can formulate $D_f(P\Vert P_0)$ as:
        \begin{equation*}
            \begin{aligned}
                D_f(P\Vert P_0) &= \int f\left( \frac{dP}{dP_0} \right) dP_0 = \int_{\{S\le t\}} f\left( \frac{P(S \le t)}{P_0(S \le t)} \right) dP_0 + \int_{\{S > t\}} f\left( \frac{P(S > t)}{P_0(S > t)} \right) dP_0 \\
                &= P_0(S\le t) f\left( \frac{P(S \le t)}{P_0(S \le t)} \right) + P_0(S > t) f\left( \frac{P(S > t)}{P_0(S > t)} \right).
            \end{aligned}
        \end{equation*}
        Let $z = P(S \le t)$ and $\beta = P_0(S \le t)$, then:
        \begin{equation*}
            D_f(P\Vert P_0) \le \rho \Longleftrightarrow \beta f\left( \frac{z}{\beta} \right) + (1-\beta) f\left( \frac{1-z}{1-\beta} \right) \le \rho \Longleftrightarrow h(z,\beta) \le \rho,
        \end{equation*}
        which matches the expression in the definition of $g_{f,\rho}$. Let $\Delta^{d-1} = \{ \pmb{\lambda} = (\lambda_, \cdots, \lambda_d) | \lambda_i \ge 0, \forall i \in [d]; \sum_{i=1}^d \lambda_i = 1 \}$ be the $(d-1)$-dimension simplex, we can analogously define:
        \begin{equation*}
            \mathcal{P}_{f,\rho}^* \!=\!\! \left\{ \!P \bigg| \exists \pmb{\lambda} \!\in\! \Delta^{d-1} \text{ s.t. } D_f\!\left(\!\!P\bigg\Vert \!\sum_{i=1}^d \!\lambda_i S_i\!\!\right) \!\!\le\! \rho, \frac{dP}{d\sum_{i=1}^d \!\lambda_i S_i} \text{ is constant on } \{ S \!\!\le\!\! t \} \text{ and } \{ S\!\! >\!\! t \} \!\!\right\}
        \end{equation*}
        and
        \begin{equation} \label{eq::worst-case-Fsub-g-relation}
            \begin{aligned}
                \widetilde{F}(t; \mathcal{P}_{f,\rho}^*)&=\inf\left\{ P(S\le t) | P \in \mathcal{P}_{f,\rho}^* \right\} \\
                &= \inf \left\{ P(S\le t) \bigg| \exists \pmb{\lambda} \in \Delta^{d-1} \text{ s.t. } D_f\left(P\bigg\Vert \sum_{i=1}^d \lambda_i S_i\right) \le \rho,  \right. \\
                &\ \ \ \ \ \ \ \ \ \ \ \ \ \ \ \ \ \ \ \ \ \ \ \ \ \ \ \ \ \ \ \ \ \ \ \ \ \left. \frac{dP}{d\sum_{i=1}^d \lambda_i S_i} \text{ is constant on } \{ S \le t \} \text{ and } \{ S > t \} \right\}\\
                &= \inf \left\{ \!P(S\le t) \bigg| \exists \pmb{\lambda} \in \Delta^{d-1} \text{ s.t. } \frac{dP}{d\sum_{i=1}^d \lambda_i S_i} \text{ is constant on } \{ S \!\le\! t \} \text{ and } \{ S \!>\! t \},  \right. \\
                &\left. \left( \sum_{i=1}^d \lambda_i F_i(t) \!\!\right)f\left(\! \frac{P(S\le t)}{\sum_{i=1}^d \lambda_i F_i(t)} \!\right) + \left(\!1\!-\!\sum_{i=1}^d \lambda_i F_i(t) \!\!\right) f\left(\! \frac{1-  P(S\le t)}{1 \!-\! \sum_{i=1}^d \lambda_i F_i(t)} \!\!\right) \!\le\! \rho \!\right\} \\
                &= \inf \left\{ P(S\le t) \bigg| \underset{\pmb{\lambda} \in \Delta^{d-1}}{\inf} h\left( P(S\le t), \sum_{i=1}^d \lambda_i F_i(t) \right) \le \rho \right\} \\
                &= g_{f,\rho}\left(F_1(t), \dots, F_d(t) \right).
            \end{aligned}
        \end{equation}
        So we have that $\widetilde{F}(t; \mathcal{P}_{f,\rho}) \le \widetilde{F}(t; \mathcal{P}_{f,\rho}^*) = g_{f,\rho}\left(F_1(t), \dots, F_d(t) \right)$. For the inverse direction, fix $t \in \mathbb{R}$, for any $P \in \mathcal{P}_{f,\rho}$, according to the definition of $\mathcal{P}_{f,\rho}$, there exists $P_0 \in \mathcal{CH}_s$ such that $D_f(P \Vert P_0) \le \rho$. Consider the following Markov kernel $K$:
        \begin{equation*}
            K(ds^\prime | s) = \left\{
                                    \begin{array}{lcl}
                                    \frac{dP_0(s^\prime) \mathbb{I}\{ s^\prime \le t \}}{P_0(S\le t)} & & {s \le t}\\
                                    \frac{dP_0(s^\prime) \mathbb{I}\{ s^\prime > t \}}{P_0(S > t)} & & {s > t}
                                    \end{array} \right..
        \end{equation*}
        Let $P_1 = K \cdot P$, then we have:
        \begin{equation*}
            \begin{aligned}
                dP_1(s) &= P_1(ds) \overset{\sroman{1}}{=} \int_{-\infty}^{+\infty} K(ds|y) dP(y) \\
                &= \int_{-\infty}^t \frac{dP_0(s) \mathbb{I}\{ s \le t \}}{P_0(S\le t)} dP(y) + \int_t^{+\infty} \frac{dP_0(s) \mathbb{I}\{ s > t \}}{P_0(S > t)} dP(y)\\
                &= \frac{dP_0(s) \mathbb{I}\{ s \le t \}}{P_0(S\le t)} P(S\le t) + \frac{dP_0(s) \mathbb{I}\{ s > t \}}{P_0(S > t)} P(S > t) \\
                &= \left( \frac{P(S\le t)}{P_0(S\le t)} \mathbb{I}\{ s \le t \} + \frac{P(S > t)}{P_0(S > t)} \mathbb{I}\{ s > t \} \right) dP_0(s),
            \end{aligned}
        \end{equation*}
        where $\sroman{1}$ comes from the definition of transition kernels. So $\frac{dP_1}{dP_0} = \frac{P(S\le t)}{P_0(S\le t)} \mathbb{I}\{ s \le t \} + \frac{P(S > t)}{P_0(S > t)} \mathbb{I}\{ s > t \}$. Similarly, let $P_2 = K\cdot P_0$, we have: $\frac{dP_2}{dP_0} = \frac{P_0(S\le t)}{P_0(S\le t)} \mathbb{I}\{ s \le t \} + \frac{P_0(S > t)}{P_0(S > t)} \mathbb{I}\{ s > t \} = 1$, so $P_2 = P_0$. Furthermore,
        \begin{equation*}
            \begin{aligned}
                P_1(S\le t) &= \int_{\{S \le t\}} dP_1 = \int_{\{S \le t\}} \left( \frac{P(S\le t)}{P_0(S\le t)} \mathbb{I}\{ s \le t \} + \frac{P(S > t)}{P_0(S > t)} \mathbb{I}\{ s > t \} \right) dP_0 \\
                &= \frac{P(S\le t)}{P_0(S\le t)} P_0(S\le t) = P(S\le t).
            \end{aligned}
        \end{equation*}
        By \Cref{lma::f-divergence-data-processing-ieq}, we have: $D_f(P_1 \Vert P_0) = D_f(P_1 \Vert P_2) \le D_f(P \Vert P_0) \le \rho$. In conclusion, for any $P \in \mathcal{P}_{f,\rho}$, we can find $P_1$ such that: there exists $P_0 \in \mathcal{CH}_s$ such that $D_f(P_1 \Vert P_0) \le \rho$ and $\frac{dP_1}{dP_0}$ is constant on $\{S \le t\}$, $P_1(S\le t) = P(S\le t)$ and $\{S > t\}$. In other words, for any $P \in \mathcal{P}_{f,\rho}$, we can find $P_1 \in \mathcal{P}_{f,\rho}^*$ such that $P_1(S\le t) = P(S\le t)$, which means that:
        \begin{equation*}
            \inf\left\{ P(S\le t) | P \in \mathcal{P}_{f,\rho}^* \right\} \le \inf\left\{ P(S\le t) | P \in \mathcal{P}_{f,\rho} \right\}.
        \end{equation*}
        On the other hand, since $\mathcal{P}_{f,\rho}^* \subseteq \mathcal{P}_{f,\rho}$, we have:
        \begin{equation} \label{eq::worst-case-Fsub-F-relation}
            \inf\left\{ P(S\le t) | P \in \mathcal{P}_{f,\rho} \right\} \le \inf\left\{ P(S\le t) | P \in \mathcal{P}_{f,\rho}^* \right\}.
        \end{equation}
        So we have $\widetilde{F}(t; \mathcal{P}_{f,\rho}) = \widetilde{F}(t; \mathcal{P}_{f,\rho}^*)$. Combining \Cref{eq::worst-case-Fsub-g-relation} with \Cref{eq::worst-case-Fsub-F-relation} implies that: $\widetilde{F}(t; \mathcal{P}_{f,\rho}) = g_{f,\rho}\left(F_1(t), \dots, F_d(t) \right)$.
        \item \textbf{When $F_i(t) = 1$ for all $i \in [d]$}. Then for any $z \in \left( g_{f,\rho}(1, \dots, 1), 1 \right]$, let $S_{z,j} \coloneqq (1-z) \delta_{t+1} + z S_j$, since. According to the proof of \Cref{lma::multi-g-reduce-to-single-g}, $h(z,1) = f(z)$ is non-increasing in $[0,1]$, so $z > g_{f,\rho}(1)$ implies that $f(z) \le \rho$. So we have:
        \begin{equation*}
            \begin{aligned}
                D_f(S_{z,j} \Vert S_j) &= \int f\left( \frac{d(1-z)\delta_{t+1}+zS_j}{dS_j} \right) dS_j \\
                &= \int_{-\infty}^{+\infty} f\left( \frac{d(1-z)\delta_{t+1}}{dS_j} + z \right) dS_j(y) \\
                &= \int_{-\infty}^t f\left( \frac{d(1-z)\delta_{t+1}}{dS_j} + z \right) dS_j(y) \\
                &= \int_{-\infty}^t f(z) dS_j(y) = f(z) \le \rho.
            \end{aligned}
        \end{equation*}
        Moreover, $S_{z,j}(S\le t) = (1-z)\delta_{t+1}(S\le t) + z S_j(S\le t) = z S_j(S\le t) = z$. So for any $z > g_{f,\rho}\left(F_1(t), \dots, F_d(t) \right)$, there exists $S_{z,j} \in \mathcal{P}_{f,\rho}$ such that $S_{z,j}(S \le t) = z$, which means that $\widetilde{F}(t; \mathcal{P}_{f,\rho}) \le g_{f,\rho}\left(F_1(t), \dots, F_d(t) \right)$. For another direction, we follow the same argument except that we use the following Markov kernel to account for the fact that $\sum_{i=1}^d \lambda_i S_i(S > t) = 0$ for any $\pmb{\lambda} \in \Delta^{d-1}$:
        \begin{equation*}
            K(ds^\prime | s) = \left\{
                                    \begin{array}{lcl}
                                    \frac{dP_0(s^\prime) \mathbb{I}\{ s^\prime \le t \}}{P_0(S\le t)} & & {s \le t}\\
                                    \delta_{s^\prime = t+1} & & {s > t}
                                    \end{array} \right..
        \end{equation*}
        \item \textbf{When $F_i(t) > 0$ for all $i \in [d]$ and there exists at least $1$ and at most $d-1$ numbers of $i\in [d]$ such that $F_i(t) = 1$}. Then we know that $0 < \underset{1 \le i \le d}{\min} F_i(t) < 1$, without loss of generality, we assume that $j = \underset{1 \le i \le d}{\arg \min} F_i(t)$. We define $\Delta^\prime = \left\{ \pmb{\lambda} | \pmb{\lambda} \in \Delta^{d-1}, \sum_{i=1}^d \lambda_i F_i(t) < 1 \right\}$, then we define:
        \begin{equation*}
            \mathcal{P}_{f,\rho}^\prime \!=\!\! \left\{ \!P \bigg| \exists \pmb{\lambda} \!\in\! \Delta^\prime \text{ s.t. } D_f\!\left(\!\!P\bigg\Vert \!\sum_{i=1}^d \!\lambda_i S_i\!\!\right) \!\!\le\! \rho, \frac{dP}{d\sum_{i=1}^d \!\lambda_i S_i} \text{ is constant on } \{ S \!\!\le\!\! t \} \text{ and } \{ S\!\! >\!\! t \} \!\!\right\}.
        \end{equation*}
        It's obvious that $\mathcal{P}_{f,\rho}^\prime \subseteq \mathcal{P}_{f,\rho}^*$ and similar to the proof in situation (2), we have:
        \begin{equation} \label{eq::worst-case-Fsub-prime} 
            \begin{aligned}
                \widetilde{F}(t; \mathcal{P}_{f,\rho}^\prime)&=\inf\left\{ P(S\le t) | P \in \mathcal{P}_{f,\rho}^\prime \right\} \\
                &= \inf \left\{ P(S\le t) \bigg| \exists \pmb{\lambda} \in \Delta^\prime \text{ s.t. } D_f\left(P\bigg\Vert \sum_{i=1}^d \lambda_i S_i\right) \le \rho,  \right. \\
                &\ \ \ \ \ \ \ \ \ \ \ \ \ \ \ \ \ \ \ \ \ \ \ \ \ \ \ \ \ \ \ \ \ \ \ \ \ \left. \frac{dP}{d\sum_{i=1}^d \lambda_i S_i} \text{ is constant on } \{ S \le t \} \text{ and } \{ S > t \} \right\}\\
                &= \inf \left\{ \!P(S\le t) \bigg| \exists \pmb{\lambda} \in \Delta^\prime \text{ s.t. } \frac{dP}{d\sum_{i=1}^d \lambda_i S_i} \text{ is constant on } \{ S \!\le\! t \} \text{ and } \{ S \!>\! t \},  \right. \\
                &\left. \left( \sum_{i=1}^d \lambda_i F_i(t) \!\!\right)f\left(\! \frac{P(S\le t)}{\sum_{i=1}^d \lambda_i F_i(t)} \!\right) + \left(\!1\!-\!\sum_{i=1}^d \lambda_i F_i(t) \!\!\right) f\left(\! \frac{1-  P(S\le t)}{1 \!-\! \sum_{i=1}^d \lambda_i F_i(t)} \!\!\right) \!\le\! \rho \!\right\} \\
                &= \inf \left\{ P(S\le t) \bigg| \underset{\pmb{\lambda} \in \Delta^\prime}{\inf} h\left( P(S\le t), \sum_{i=1}^d \lambda_i F_i(t) \right) \le \rho \right\}.
            \end{aligned}
        \end{equation}
        According to the definition of $\Delta^\prime$, we have:
        \begin{equation*}
            \left\{ \sum_{i=1}^d \lambda_i F_i(t) \bigg| \pmb{\lambda} \in \Delta^\prime \right\} = \left[ F_j(t), 1 \right).
        \end{equation*}
        So we have:
        \begin{equation*}
            \widetilde{F}(t; \mathcal{P}_{f,\rho}^\prime) = \inf \left\{ P(S\le t) \bigg| \underset{\beta \in [ F_j(t), 1)}{\inf} h\left( P(S\le t), \beta \right) \le \rho \right\}.
        \end{equation*}
        A similar argument as in the proof \Cref{lma::multi-g-reduce-to-single-g} provides: $\widetilde{F}(t; \mathcal{P}_{f,\rho}^\prime) = g_{f,\rho}(F_j(t))$. Since $F_{j}(t) = \underset{1 \le i \le d}{\min} F_i(t)$, \Cref{lma::multi-g-reduce-to-single-g} tells us that $g_{f,\rho}(F_j(t)) = g_{f,\rho}\left(F_1(t), \cdots, F_d(t) \right)$, which means that $\widetilde{F}(t; \mathcal{P}_{f,\rho}) \le \widetilde{F}(t; \mathcal{P}_{f,\rho}^\prime) = g_{f,\rho}\left(F_1(t), \cdots, F_d(t) \right)$.

        The other direction comes from similar arguments in situations (2) and (3). When the $P_0$ satisfies $P_0(S \le t) < 1$, we use the Markov kernel defined in situation (2), otherwise, we use the Markov kernel defined in the situation (3).
        
     \end{enumerate}
\end{proof}

Now we begin to proceed with the proof of \Cref{thm::wc-quantile-to-standard-quantile}, which is a direct result from \Cref{lma::worst-case-F-multi-g-relation}.
\begin{proof} [Proof of \Cref{thm::wc-quantile-to-standard-quantile}]
    \begin{equation*}
        \begin{aligned}
            \widetilde{\mathcal{Q}}(\alpha; \mathcal{P}_{f,\rho}) &\overset{\sroman{1}}{=} \inf \left\{ q \in \mathbb{R} \Big| \widetilde{F}(q; \mathcal{P}_{f,\rho}) \ge \alpha \right\} \overset{\sroman{2}}{=} \inf \left\{ q \in \mathbb{R} \Big| g_{f,\rho}\left( F_{\min}(q) \right) \ge \alpha \right\} \\
            &\overset{\sroman{3}}{=} \inf \left\{ q \in \mathbb{R} \Big| F_{\min}(q) \ge g_{f,\rho}^{-1} (\alpha) \right\} = \mathcal{Q}\left( g_{f,\rho}^{-1} (\alpha); F_{\min} \right),
        \end{aligned}
    \end{equation*}
    where $\sroman{1}$ follows the definition of the quantile function; $\sroman{2}$ comes from \Cref{lma::worst-case-F-multi-g-relation}; $\sroman{3}$ is a result of item (5) in \Cref{lma::properties-of-g}.
\end{proof}

\subsection{Proof of \Cref{prop::Fmin-empirical-approx}} \label{prf::prop-Fmin-empirical-approx}
\begin{custompro}{\ref{prop::Fmin-empirical-approx}}
    Let $F_1, \dots, F_d$ be c.d.f.'s on $\mathbb{R}$, define $F_{\min}(x) = \underset{1 \le i \le d}{\min} F_i(x)$. Suppose $\hat{F}_1, \dots, \hat{F}_d$ are the empirical c.d.f.'s corresponding to $F_1, \dots, F_d$ with $m_1, \dots, m_d$ examples, respectively. Define $\hat{F}_{\min}(x) = \underset{1 \le i \le d}{\min} \hat{F}_i(x)$, then for any $\epsilon > 0$,
    \begin{equation*}
        \mathbb{P}\left( \underset{x \in \mathbb{R}}{\sup} \left| F_{\min}(x) - \hat{F}_{\min}(x) \right| > \epsilon \right) \le 2 \sum_{i=1}^d e^{-2m_i \epsilon^2},
    \end{equation*}
    where the probability is over the randomness of the examples that define the empirical c.d.f.'s.
\end{custompro}

We need the following famous Dvoretzky–Kiefer–Wolfowitz inequality (DKW inequality for short) as a tool.
\begin{lemma} [{\citep[DKW inequality]{massart1990tight}}] \label{lma::DKW}
    Given a natural number $n$, let $X_1, X_2, \dots, X_n$ be real-valued i.i.d. random variables with c.d.f. $F(\cdot)$. Let $F_n$ denote the associated empirical distribution function defined by:
    \begin{equation*}
        F_n(x) = \frac{1}{n} \sum_{i=1}^n \mathbb{I}_{(-\infty, x]}(X_i), \ \ \ \ \ \ x \in \mathbb{R}.
    \end{equation*}
    Then, for any $\epsilon \ge \sqrt{\frac{1}{2n} \ln{2}}$, we have:
    \begin{equation*}
        \mathbb{P}\left( \underset{x \in \mathbb{R}}{\sup} \left( F_n(x) - F(x) \right) > \epsilon \right) \le e^{-2n \epsilon^2},
    \end{equation*}
    moreover, for any $\epsilon > 0$, we have:
    \begin{equation*}
        \mathbb{P}\left( \underset{x \in \mathbb{R}}{\sup} \left| F_n(x) - F(x) \right| > \epsilon \right) \le 2 e^{-2n \epsilon^2}.
    \end{equation*}
\end{lemma}

\begin{proof} [Proof of \Cref{prop::Fmin-empirical-approx}]
    Fix $\epsilon > 0$, we define the events $A_i, i \in [d]$ as: $A_i = \left\{ \underset{x \in \mathbb{R}}{\sup} \left| \hat{F}_i(x) - F_i(x) \right| \le \epsilon \right\}$. Similarly, define $A = \left\{ \underset{x \in \mathbb{R}}{\sup} \left| \hat{F}_{\min}(x) - F_{\min}(x) \right| \le \epsilon \right\}$. If $\cap_{i=1}^d A_i$ holds, then, for all $i \in [d]$:
    \begin{equation*}
        \underset{x \in \mathbb{R}}{\sup} \left| \hat{F}_i(x) - F_i(x) \right| \le \epsilon,
    \end{equation*}
    which means that:
    \begin{equation*}
        \forall i \in [d], \forall x \in \mathbb{R}: \ \ F_i(x) - \epsilon \le \hat{F}_i(x) \le F_i(x) + \epsilon.
    \end{equation*}
    For an arbitrary $x \in \mathbb{R}$, assume that $j = \underset{1 \le i \le d}{\arg\min} \hat{F}_i(x)$, then:
    \begin{equation} \label{ieq::Fmin-approx-ge-directon}
        \hat{F}_{\min}(x) = \hat{F}_j(x) \ge F_j(x) - \epsilon \ge \underset{1 \le i \le d}{\min} F_i(x) - \epsilon \ge F_{\min}(x) - \epsilon.
    \end{equation}
    Suppose $k = \underset{1 \le i \le d}{\arg\min} F_i(x)$, then:
    \begin{equation} \label{ieq::Fmin-approx-le-directon}
        F_{\min}(x) + \epsilon = F_k(x) + \epsilon \ge \hat{F}_k(x) \ge \underset{1 \le i \le d}{\min} \hat{F}_i(x) = \hat{F}_{\min}(x).
    \end{equation}
    Combining \Cref{ieq::Fmin-approx-ge-directon} and \Cref{ieq::Fmin-approx-le-directon} shows that:
    \begin{equation*}
        \forall x \in \mathbb{R}: \ \ F_{\min}(x) - \epsilon \le \hat{F}_{\min}(x) \le F_{\min}(x) + \epsilon,
    \end{equation*}
    which means that:
    \begin{equation*}
        \underset{x \in \mathbb{R}}{\sup} \left| \hat{F}_{\min}(x) - F_{\min}(x) \right| \le \epsilon.
    \end{equation*}
    So we have proved that $\cap_{i=1}^d A_i \subseteq A$, which means that $A^c \subseteq \left( \cap_{i=1}^d A_i \right)^c = \cup_{i=1}^d A_i^c$, where the superscript $A^c$ is the complement set of $A$. So we have:
    \begin{equation} \label{ieq::prob-Ac-le-sum-Aic}
        \mathbb{P}\left( A^c \right) \le \mathbb{P}\left( \cup_{i=1}^d A_i^c \right) \overset{\sroman{1}}{\le} \sum_{i=1}^d \mathbb{P}\left( A_i^c \right),
    \end{equation}
    where $\sroman{1}$ comes from the subadditivity of the probability measure. \Cref{ieq::prob-Ac-le-sum-Aic} implies that:
    \begin{equation*}
        \begin{aligned}
            \mathbb{P}\left( \underset{x \in \mathbb{R}}{\sup} \left| F_{\min}(x) - \hat{F}_{\min}(x) \right| > \epsilon \right) &\le \sum_{i=1}^d \mathbb{P}\left( \underset{x \in \mathbb{R}}{\sup} \left| F_i(x) - \hat{F}_i(x) \right| > \epsilon \right)\\
            &\overset{\sroman{2}}{\le} 2 \sum_{i=1}^d e^{-2 m_i \epsilon^2}.
        \end{aligned}
    \end{equation*}
\end{proof}

\subsection{Proof of \Cref{thm::coverage-for-finite-approx}} \label{prf::thm-coverage-for-finite-approx}
\begin{customthm}{\ref{thm::coverage-for-finite-approx}} [Marginal coverage guarantee for the empirical estimations]
    Assume $S_{n+1} = s(X_{n+1}, Y_{n+1})\sim P \in \mathcal{P}_{f,\rho}$ is independent of $\{ S_{ij}\}_{i,j=1}^{d, m_i}$ where $\{ S_{ij}\}_{j=1}^{m_i} \overset{i.i.d.}{\sim} s\#S_i$ for $i \in [d]$. Suppose $\rho^\star = \underset{P_0 \in \mathcal{CH}_s}{\inf} D_f(P\Vert P_0) \le \rho$ where $\mathcal{CH}_s = \mathcal{CH}(s\# S_1, \cdots, s\# S_d)$. Let $\hat{F}_{\min}$ be defined as in \Cref{prop::Fmin-empirical-approx} and let $\hat{S}_1, \dots, \hat{S}_d$ be the empirical distributions of $S_1, \dots, S_d$ respectively. Define
    \begin{equation*}
        \hat{\mathcal{P}}_{f,\rho} \coloneqq \left\{ S \text{ is a distribution on } \mathbb{R} \Big| \exists S_0 \in \mathcal{CH}(s\# \hat{S}_1, \cdots, s\# \hat{S}_d)\ \text{s.t.} \ D_f(S\Vert S_0) \le \rho\right\}.
    \end{equation*}
    If we set $t = \widetilde{\mathcal{Q}}(1-\alpha; \hat{\mathcal{P}}_{f,\rho}) = \mathcal{Q}\left( g_{f,\rho}^{-1}(1-\alpha) ; \hat{F}_{\min} \right)$, then for any $\epsilon > 0$, we can get the following marginal coverage guarantee for $\widetilde{\mathcal{C}}$:
    \begin{equation*}
        \begin{aligned}
            \mathbb{P}\left( Y_{n+1} \in \widetilde{\mathcal{C}} (X_{n+1}) \right) &\ge \left( 1 - 2 \sum_{i=1}^d e^{-2 m_i \epsilon^2} \right) g_{f,\rho^\star}\left( g_{f,\rho}^{-1}(1-\alpha) - \epsilon \right) \\
            &\ge \left( 1 - 2 \sum_{i=1}^d e^{-2 m_i \epsilon^2} \right)\left( 1 - \alpha - \epsilon \cdot g_{f,\rho}^\prime\left( g_{f,\rho}^{-1}(1-\alpha) \right) \right),
        \end{aligned}
    \end{equation*}    
    where the randomness is over the choice of the source examples and $(X_{n+1}, Y_{n+1})$.
\end{customthm}

\begin{proof} [Proof of \Cref{thm::coverage-for-finite-approx}]
    Suppose $F$ is the c.d.f. of the target domain $P$. Recall that by \Cref{lma::worst-case-F}, $\widetilde{F}(t; \mathcal{P}_{f,\rho}) = \underset{P \in \mathcal{P}_{f,\rho}}{\inf} P(S\le t)$, then we have, for any $t \in \mathbb{R}$:
    \begin{equation*}
        F(t) \ge \widetilde{F}(t; \mathcal{P}_{f,\rho}) \overset{\sroman{1}}{=} g_{f,\rho^\star}\left(F_{\min}(t)\right),
    \end{equation*}
    where $\sroman{1}$ is from \Cref{lma::worst-case-F-multi-g-relation}. Recall that $\widetilde{\mathcal{C}}(x) = \{ y \in \mathcal{Y} | s(x,y) \le t \}$, set $t = \widetilde{\mathcal{Q}}(1-\alpha; \hat{\mathcal{P}}_{f,\rho})$, then we have:
    \begin{equation*}
        \begin{aligned}
            \mathbb{P}&\left( Y_{n+1}  \in \widetilde{\mathcal{C}}(X_{n+1}) \Big| \{ (X_{ij}, Y_{ij}) \}_{i,j=1}^{d,m_i} \right) \overset{\sroman{2}}{=} \mathbb{P}\left( s(X_{n+1}, Y_{n+1}) \le \widetilde{\mathcal{Q}}(1-\alpha; \hat{\mathcal{P}}_{f,\rho})\Big| \{ (X_{ij}, Y_{ij}) \}_{i,j=1}^{d,m_i} \right) \\
            &\overset{\sroman{3}}{=} F\left( \widetilde{\mathcal{Q}}(1-\alpha; \hat{\mathcal{P}}_{f,\rho}) \right) \ge g_{f,\rho^\star}\left(F_{\min}\left( \widetilde{\mathcal{Q}}(1-\alpha; \hat{\mathcal{P}}_{f,\rho}) \right)\right) \\
            &\overset{\sroman{4}}{=} g_{f,\rho^\star}\left( F_{\min} \left( \mathcal{Q}\left( g_{f,\rho}^{-1}(1-\alpha) ; \hat{F}_{\min} \right) \right) \right),
        \end{aligned}
    \end{equation*}
    where $\sroman{2}$ is from the definition of $\widetilde{\mathcal{C}}$ and $t$; $\sroman{3}$ is from the fact that $(X_{n+1}, Y_{n+1})$ is independent of $\{ (X_{ij}, Y_{ij}) \}_{i,j=1}^{d,m_i}$ and $\sroman{4}$ is a result of \Cref{thm::wc-quantile-to-standard-quantile}. For $\epsilon > 0$, when $\underset{x \in \mathbb{R}}{\sup} \left| F_{\min}(x) - \hat{F}_{\min}(x) \right| \le \epsilon$, we have:
    \begin{equation*}
        F_{\min} \left( \mathcal{Q}\left( g_{f,\rho}^{-1}(1-\alpha) ; \hat{F}_{\min} \right) \right) \ge \hat{F}_{\min} \left( \mathcal{Q}\left( g_{f,\rho}^{-1}(1-\alpha) ; \hat{F}_{\min} \right) \right) - \epsilon = g_{f,\rho}^{-1}(1-\alpha) - \epsilon.
    \end{equation*}
    
    For any fixed $\epsilon>0$, define $A = \left\{ \underset{x \in \mathbb{R}}{\sup} \left| F_{\min}(x) - \hat{F}_{\min}(x) \right| \le \epsilon \right\}$ and let $P^*$ be the probability measure with respect to $\{ (X_{ij}, Y_{ij}) \}_{i,j=1}^{d,m_i}$. Taking expectation with respect to $\{ (X_{ij}, Y_{ij}) \}_{i,j=1}^{d,m_i}$ yields:
    \begin{equation*}
        \begin{aligned}
            \mathbb{P}&\left( Y_{n+1} \in \widetilde{\mathcal{C}} (X_{n+1}) \right) = \underset{\{ (X_{ij}, Y_{ij}) \}}{\mathbb{E}} \left[ \mathbb{P} \left( Y_{n+1}  \in \widetilde{\mathcal{C}}(X_{n+1}) \Big| \{ (X_{ij}, Y_{ij}) \}_{i,j=1}^{d,m_i} \right) \right] \\
            &= \int_A \mathbb{P} \left( Y_{n+1}  \in \widetilde{\mathcal{C}}(X_{n+1}) \Big| \{ (X_{ij}, Y_{ij}) \}_{i,j=1}^{d,m_i} \right) dP^* \\
            &\ \ \ \ \ \ \ \ \ \ \ \ \ \ \ \ \ \ \ \ \ \ \ \ \ \ \ \ \ + \int_{A^c} \mathbb{P} \left( Y_{n+1}  \in \widetilde{\mathcal{C}}(X_{n+1}) \Big| \{ (X_{ij}, Y_{ij}) \}_{i,j=1}^{d,m_i} \right) dP^* \\
            &\ge \int_A \mathbb{P} \left( Y_{n+1}  \in \widetilde{\mathcal{C}}(X_{n+1}) \Big| \{ (X_{ij}, Y_{ij}) \}_{i,j=1}^{d,m_i} \right) dP^* \\
            &\ge \int_A g_{f,\rho^\star}\left( F_{\min} \left( \mathcal{Q}\left( g_{f,\rho}^{-1}(1-\alpha) ; \hat{F}_{\min} \right) \right) \right) dP^* \\
            &\overset{\sroman{5}}{\ge} \int_A g_{f,\rho^\star}\left( g_{f,\rho}^{-1}(1-\alpha) - \epsilon \right) dP^* = \mathbb{P}(A)\  g_{f,\rho^\star}\left( g_{f,\rho}^{-1}(1-\alpha) - \epsilon \right) \\
            &\overset{\sroman{6}}{\ge} \left( 1 - 2 \sum_{i=1}^d e^{-2 m_i \epsilon^2} \right) g_{f,\rho^\star}\left( g_{f,\rho}^{-1}(1-\alpha) - \epsilon \right)
        \end{aligned}
    \end{equation*}
    where $\sroman{5}$ comes from the fact that $g_{f,\rho}(\beta)$ is non-decreasing in $\beta$ (\Cref{lma::properties-of-g}) and $\sroman{6}$ is a result of \Cref{prop::Fmin-empirical-approx}. Since $g_{f,\rho}(\beta)$ is non-increasing (\Cref{lma::properties-of-g}) in $\rho$ and $\rho^\star \le \rho$, we know that:
    \begin{equation*}
        g_{f,\rho^\star}\left( g_{f,\rho}^{-1}(1-\alpha) - \epsilon \right) \ge g_{f,\rho}\left( g_{f,\rho}^{-1}(1-\alpha) - \epsilon \right).
    \end{equation*}
    \Cref{lma::properties-of-g} shows that $g_{f,\rho}(\beta)$ is convex with respect to $\beta$, so for any $x,y \in [0,1]$:
    \begin{equation*}
        g_{f,\rho}(y) \ge g_{f,\rho}(x) + g_{f,\rho}^\prime(x)(y-x).
    \end{equation*}
    Taking $x = g_{f,\rho}^{-1}(1-\alpha)$ and $y = g_{f,\rho}^{-1}(1-\alpha) - \epsilon$ yields:
    \begin{equation*}
        \begin{aligned}
            g_{f,\rho}\left( g_{f,\rho}^{-1}(1-\alpha) - \epsilon \right) &\ge g_{f,\rho}\left( g_{f,\rho}^{-1}(1-\alpha) \right) + g_{f,\rho}^{\prime}\left( g_{f,\rho}^{-1}(1-\alpha)\right)\cdot \epsilon \\
            &= 1 - \alpha - \epsilon \cdot g_{f,\rho}^{\prime}\left( g_{f,\rho}^{-1}(1-\alpha)\right).
        \end{aligned}
    \end{equation*}
\end{proof}

\subsection{Proof of \Cref{cor::correct-coverage-guarantee}} \label{prf::cor-correct-coverage-guarantee}
\begin{customcor}{\ref{cor::correct-coverage-guarantee}} [Correct the prediction set to get a $(1-\alpha)$ marginal coverage guarantee]
    Let $(X_{n+1}, Y_{n+1})$, $\hat{F}_{\min}$, $\hat{\mathcal{P}}_{f,\rho}$ be defined as in \Cref{thm::coverage-for-finite-approx}. For arbitrary $\epsilon > 0$, if we set $t = \widetilde{\mathcal{Q}}(1-\alpha^\prime; \hat{\mathcal{P}}_{f,\rho}) = \mathcal{Q}\left( g_{f,\rho}^{-1}(1-\alpha^\prime) ; \hat{F}_{\min} \right)$, where
    \begin{equation*}
        \alpha^\prime = 1 - g_{f,\rho}\left( \epsilon + g_{f,\rho}^{-1}\left( \frac{1-\alpha}{1 - 2 \sum_{i=1}^d e^{-2 m_i \epsilon^2}} \right) \right),
    \end{equation*}
    then we can get the following marginal coverage guarantee:
    \begin{equation*}
         \mathbb{P}\left( Y_{n+1} \in \widetilde{\mathcal{C}} (X_{n+1}) \right) \ge 1 - \alpha.
    \end{equation*}
\end{customcor}

\begin{proof} [Proof of \Cref{cor::correct-coverage-guarantee}]
    By \Cref{lma::properties-of-g}, $g_{f,\rho}(\beta)$ is non-increasing in $\rho$, since $\rho^\star \le \rho$, $g_{f,\rho^\star}(\beta) \ge g_{f,\rho}(\beta)$ for any $\beta \in [0,1]$. Then by \Cref{thm::coverage-for-finite-approx} we know that:
    \begin{equation*}
        \begin{aligned}
            \mathbb{P}&\left( Y_{n+1} \in \widetilde{\mathcal{C}} (X_{n+1}) \right) \ge \left( 1 - 2 \sum_{i=1}^d e^{-2 m_i \epsilon^2} \right) g_{f,\rho^\star}\left( g_{f,\rho}^{-1}(1-\alpha^\prime) - \epsilon \right) \\
            &\ge \left( 1 - 2 \sum_{i=1}^d e^{-2 m_i \epsilon^2} \right) g_{f,\rho}\left( g_{f,\rho}^{-1}(1-\alpha^\prime) - \epsilon \right).
        \end{aligned}
    \end{equation*}
    Since $\alpha^\prime = 1 - g_{f,\rho}\left( \epsilon + g_{f,\rho}^{-1}\left( \frac{1-\alpha}{1 - 2 \sum_{i=1}^d e^{-2 m_i \epsilon^2}} \right) \right)$, we have:
    \begin{equation*}
        \begin{aligned}
            \mathbb{P}&\left( Y_{n+1} \in \widetilde{\mathcal{C}} (X_{n+1}) \right) \ge \left( 1 - 2 \sum_{i=1}^d e^{-2 m_i \epsilon^2} \right) g_{f,\rho}\left( g_{f,\rho}^{-1}(1-\alpha^\prime) - \epsilon \right) \\
            &= \left( 1 - 2 \sum_{i=1}^d e^{-2 m_i \epsilon^2} \right) g_{f,\rho}\left( g_{f,\rho}^{-1}\left(g_{f,\rho}\left( \epsilon + g_{f,\rho}^{-1}\left( \frac{1-\alpha}{1 - 2 \sum_{i=1}^d e^{-2 m_i \epsilon^2}} \right) \right)\right) - \epsilon \right)\\
            &= \left( 1 - 2 \sum_{i=1}^d e^{-2 m_i \epsilon^2} \right) g_{f,\rho}\left( \epsilon + g_{f,\rho}^{-1}\left( \frac{1-\alpha}{1 - 2 \sum_{i=1}^d e^{-2 m_i \epsilon^2}} \right) - \epsilon \right) \\
            &= 1- \alpha.
        \end{aligned}
    \end{equation*}
\end{proof}

\subsection{Proof of \Cref{lma::g-inverse-form}} \label{prf::lma-g-inverse-form}

\begin{customlemma}{\ref{lma::g-inverse-form}} [The form of $g_{f,\rho}^{-1}$ that can be efficiently solved]
    Let $g_{f,\rho}, g_{f,\rho}^{-1}$ be defined as in \Cref{thm::wc-quantile-to-standard-quantile}, then for $\tau \in [0,1]$, we have:
    \begin{equation*}
        g_{f,\rho}^{-1}(\tau) = \sup \left\{ \beta \in [\tau, 1]\left| \beta f\left( \frac{\tau}{\beta} \right) + (1-\beta) f\left( \frac{1-\tau}{1-\beta} \right) \le \rho \right. \right\}.
    \end{equation*}
\end{customlemma}

\begin{proof} [Proof of \Cref{lma::g-inverse-form}]
    As in the proof of \Cref{lma::multi-g-reduce-to-single-g}, we define $h(\beta,z) = \beta f\left( \frac{z}{\beta} \right) + (1-\beta) f \left( \frac{1-z}{1-\beta} \right)$ for simplicity, moreover, $h(\beta,z)$ is convex.
    
    For any $\rho >0$, it is obvious that $h(\tau,\tau) = 0 < \rho$, so $g_{f,\rho}(\tau) \le \tau$, which means that $g_{f,\rho}^{-1}(\tau) \ge \tau$. So $g_{f,\rho}^{-1}(\tau) = \sup \left\{ \beta \in [\tau,1]\left| g_{f,\rho}(\beta) \le \tau \right. \right\}$. Moreover, for $\beta \in [\tau, 1]$, since the minimum of $h(\beta,z)$ is achieved at $z=\beta$ for a given $\beta$, we know that $\underset{0 \le z \le \tau}{\inf} h(\beta,z) = h(\beta,\tau)$. Then we have:

    \begin{equation*}
        \begin{aligned}
            g_{f,\rho}^{-1}(\tau) &= \sup \left\{ \beta \in [\tau,1]\left| \inf \left\{ z \in [0,1] \left| h(\beta, z) \le \rho \right. \right\} \le \tau \right. \right\} \\
            &\overset{\sroman{1}}{=} \sup \left\{ \beta \in [\tau,1]\left| \exists z \le \tau \text{ s.t. } h(\beta, z) \le \rho \right. \right\} \\
            &= \sup \left\{ \beta \in [\tau,1]\left| \underset{0\le z \le \tau}{\inf} h(\beta, z) \le \rho \right. \right\} \\
            &\overset{\sroman{2}}{=} \sup \left\{ \beta \in [\tau,1]\left| h(\beta, \tau) \le \rho \right. \right\},
        \end{aligned}
    \end{equation*}
    where $\sroman{1}$ is from the fact that: when $\beta$ is fixed, $h(z,\beta)$ is continuous in $z$ and the infimum can be achieved and $\sroman{2}$ is implied by the fact that $\underset{0 \le z \le \tau}{\inf} h(\beta,z) = h(\beta,\tau)$ when $\beta \in [\tau,1]$.
\end{proof}

\subsection{Proof of the examples} \label{prf::examples}
\begin{proof} [Proof of \Cref{example::chi-square-div}]
    For $g_{f,\rho}$, when $f(t) = (t-1)^2$, we have that, for $0 < \beta < 1$:
    \begin{equation*}
        \begin{aligned}
            g_{f,\rho}(\beta) &= \inf \left\{ z\in [0,1] \left| \beta f\left( \frac{z}{\beta} \right) + (1-\beta) f\left( \frac{1-z}{1-\beta} \right) \le \rho \right. \right\} \\
            &= \inf \left\{ z\in [0,1] \left| \beta \left( \frac{z}{\beta} - 1 \right)^2 + (1-\beta) \left( \frac{1-z}{1-\beta} -1 \right)^2 \le \rho \right. \right\} \\
            &= \inf \left\{ z\in [0,1] \left| \left( z- \beta \right)^2 \le \rho\beta (1-\beta) \right. \right\}\\
            &= \left( \beta - \sqrt{\rho \beta (1-\beta)} \right)_+.
        \end{aligned}
    \end{equation*}
    When $\beta \downarrow 0$, $\beta f\left( \frac{z}{\beta} \right) = \left( \frac{z}{\sqrt{\beta}} - \sqrt{\beta} \right)^2 \to + \infty$ unless $z=0$. By \Cref{lma::properties-of-g}, when $\rho > 0$, $g_{f,\rho}(\beta)$ is continuous for $\beta \in [0,1]$, so we have $g_{f,\rho}(0) = \underset{\beta \downarrow 0}{\lim} g_{f,\rho}(\beta) = 0$. Similarly, we have: $g_{f,\rho}(1) = 1$. Since the value of $g_{f,\rho}$ on $0,1$ is consistent with the formula $g_{f,\rho}(\beta) = \left( \beta - \sqrt{\rho \beta (1-\beta)} \right)_+$, we conclude that $g_{f,\rho}(\beta) = \left( \beta - \sqrt{\rho \beta (1-\beta)} \right)_+$ for $\beta \in [0,1]$.

    For $g_{f,\rho}^{-1}$, we first solve the equation $\beta - \sqrt{\rho \beta (1-\beta)}=0$ and get a solution $\beta^\star = \frac{\rho}{\rho+1}$. By \Cref{lma::properties-of-g}, $g_{f,\rho}(\beta)$ is non-decreasing and continuous when $\beta \in [0,1]$, so $g_{f,\rho}(\beta) = \beta - \sqrt{\rho \beta (1-\beta)}$ when $\beta \ge \beta^\star$. By the definition of $g_{f,\rho}^{-1}$, it's obvious that $\beta^\star \le g_{f,\rho}^{-1}(\tau)$, so we can compute $g_{f,\rho}^{-1}(\tau)$ by solving the following optimization problem:
    \begin{equation*}
        \max \beta \text{   s.t. } \left\{
                                        \begin{array}{l}
                                        \frac{\rho}{\rho + 1} \le \beta \le 1 \\
                                        \beta - \sqrt{\rho \beta(1-\beta)} \le \tau
                                        \end{array} \right..
    \end{equation*}
\end{proof}

\begin{proof} [Proof of \Cref{example::TV}]
    For $g_{f,\rho}$, when $f(t) = \frac{1}{2} |t-1|$, we have that, for $0 < \beta < 1$:
    \begin{equation*}
        \begin{aligned}
            g_{f,\rho}(\beta) &= \inf \left\{ z\in [0,1] \left| \beta f\left( \frac{z}{\beta} \right) + (1-\beta) f\left( \frac{1-z}{1-\beta} \right) \le \rho \right. \right\} \\
            &= \inf \left\{ z\in [0,1] \left| \frac{1}{2}\beta \left| \frac{z}{\beta} -1 \right| + \frac{1}{2}(1-\beta) \left| \frac{1-z}{1-\beta} - 1 \right| \le \rho \right. \right\} \\
            &= \inf \left\{ z\in [0,1] \left|  \left| z -\beta \right|  \le \rho \right. \right\} \\
            &= \left( \beta - \rho \right)_+.
        \end{aligned}
    \end{equation*}
    According to \Cref{lma::properties-of-g}, $g_{f,\rho}(\beta)$ is continuous when $\beta \in [0,1]$, so $g_{f,\rho}(\beta) = \left( \beta - \rho \right)_+$ when $\beta \in [0,1]$.

    For $g_{f,\rho}^{-1}$, we have:
    \begin{equation*}
        g_{f,\rho}^{-1}(\tau) = \sup \left\{ \beta \in [0,1] \left| g_{f,\rho}(\beta) \le \tau \right. \right\} = \sup \left\{ \beta \in [0,1] \left| \left( \beta - \rho \right)_+ \le \tau \right. \right\} = \min \{ \tau + \rho, 1 \}.
    \end{equation*}
\end{proof}

\section{B Additional experimental results} \label{apd::extra-exp-results}
In this section, we provide additional experimental results. We provide the results of the average length of the predicted confidence sets in \Cref{fig::exp-source1-efficiency,fig::exp-source2-efficiency}, corresponding to the results of the coverage in \Cref{fig::exp-source1-coverage,fig::exp-source2-coverage}, respectively.

\begin{figure}
  \centering
  \includegraphics[width=1\textwidth]{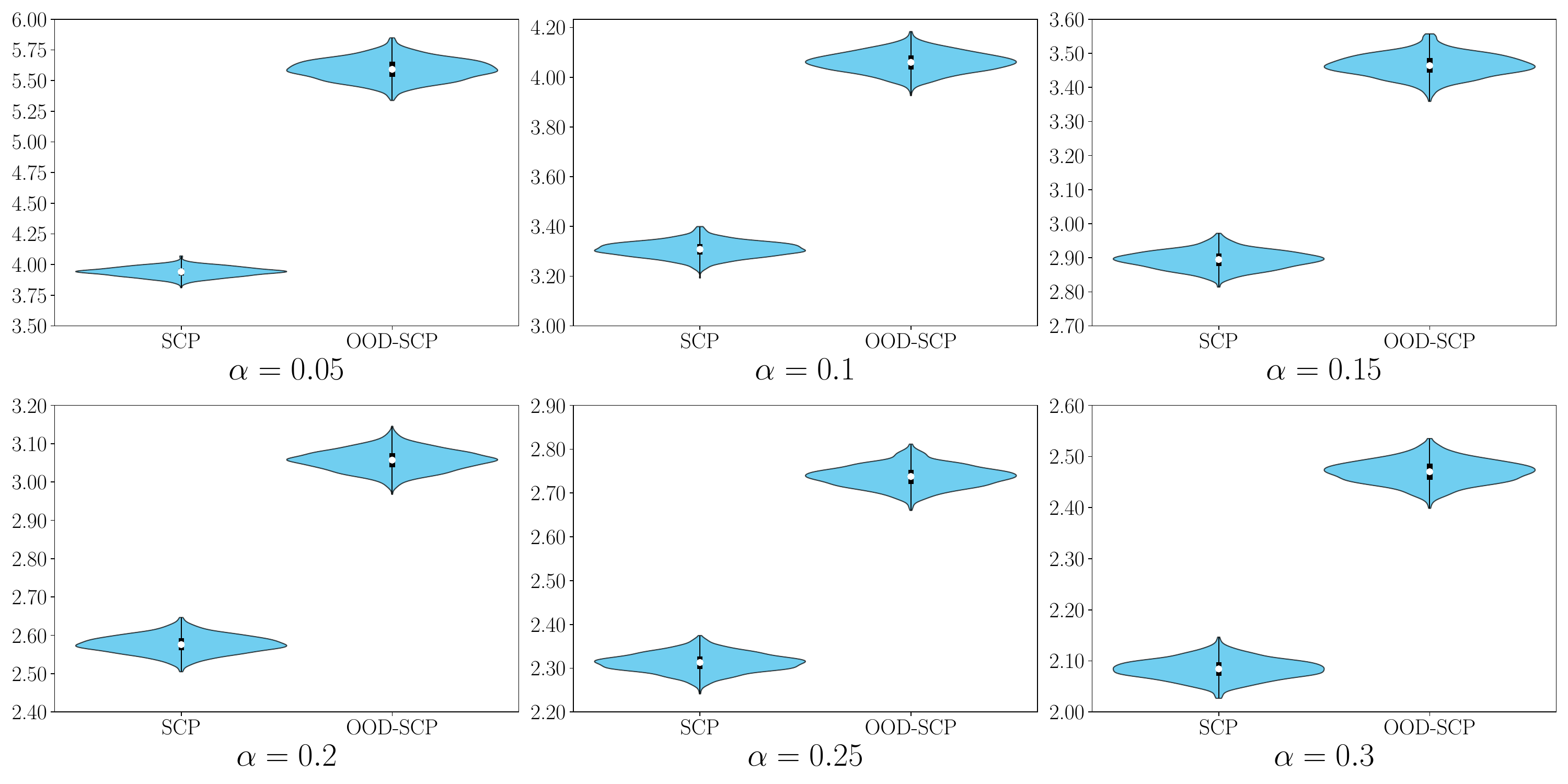}    
  \caption{The violin plots for the \textbf{average length} of the $1000$ runs under the same data generation settings as in Section \ref{sec::motivating-experiment}. We show results for $\alpha=\{0.05, 0.1, 0.15, 0.2, 0.25, 0.3\}$. The white point represents the median, while the two endpoints of the thick line are the $0.25$ quantile and the $0.75$ quantile.} \label{fig::exp-source1-efficiency}
\end{figure}

\begin{figure}
  \centering
  \includegraphics[width=1\textwidth]{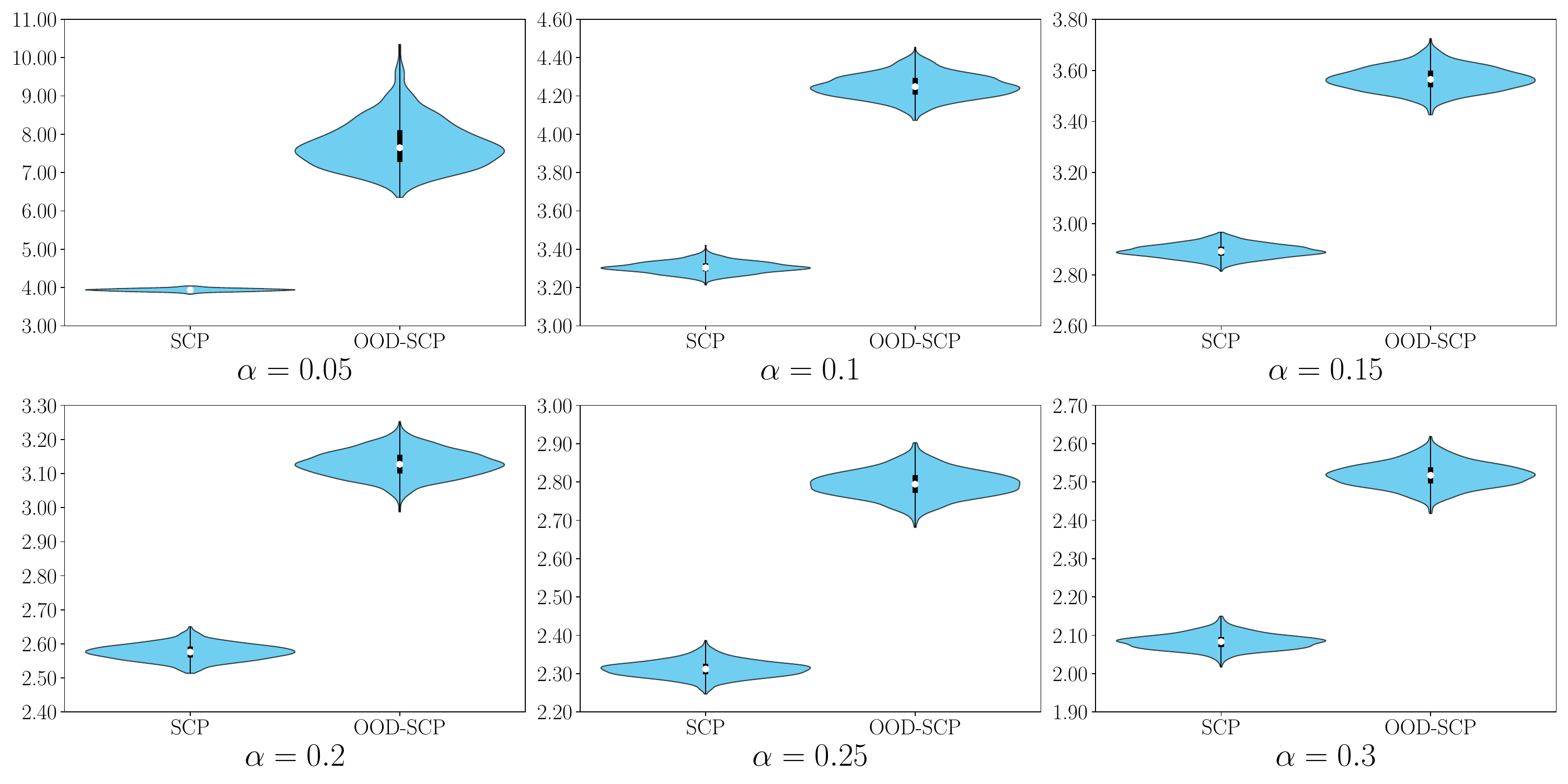}    
  \caption{The violin plots for the \textbf{average length} of the $1000$ runs for the multi-source OOD confidence set prediction task. We show results for $\alpha=\{0.05, 0.1, 0.15, 0.2, 0.25, 0.3\}$. The white point represents the median, while the two endpoints of the thick line are the $0.25$ quantile and the $0.75$ quantile.} \label{fig::exp-source2-efficiency}
\end{figure}

\end{document}